\documentclass[twoside]{article}

\usepackage{graphicx}    
\usepackage{float}       
\usepackage{amssymb} 
\usepackage{subcaption}
\usepackage{hyperref}

\usepackage{booktabs}
\usepackage{caption}
\usepackage{algorithm}
\usepackage{algcompatible}
\usepackage{algpseudocode}
\usepackage{placeins}
\usepackage{hyperref}       

\usepackage{amsmath}
\usepackage{mathtools}
\usepackage{amsthm}
\usepackage{bm}
\usepackage{xcolor}

\usepackage[accepted]{aistats2025}
\usepackage{natbib}
\usepackage{cuted}
\usepackage{xspace}

\usepackage{tikz}
\newtheorem{lemma}{Lemma}
\newtheorem{theorem}{Theorem}
\newtheorem{remark}{Remark}

\newtheorem{proposition}{Proposition}
\newtheorem{definition}{Definition}

\newtheorem{example}{Example}
\newcommand{\name}[0]{\mbox{WEGP}\xspace}

\makeatletter
\providecommand\thanks[1]{%
  \footnotemark
  \protected@xdef\@thanks{\@thanks\protect\footnotetext{#1}}%
}
\makeatother
\runningtitle{Weighted Euclidean Distance Matrices over Mixed Inputs for Gaussian Process Models}
\begin{document}

\twocolumn[
\aistatstitle{Weighted Euclidean Distance Matrices over Mixed Continuous and Categorical Inputs for Gaussian Process Models}
\aistatsauthor{Mingyu Pu$^{1,}$\textsuperscript{*}  \And Songhao Wang$^{2,}$\textsuperscript{*}\And Haowei Wang$^{1,}\textsuperscript{$\dagger$}$ \And Szu Hui Ng$^1$}
\aistatsaddress{$^1$Department of ISEM,\\
    National University of Singapore  \And $^2$College of Business,\\
     Southern University of Science and Technology}]

\begin{abstract}
Gaussian Process (GP) models are widely utilized as surrogate models in scientific and engineering fields. However, standard GP models are limited to continuous variables due to the difficulties in establishing correlation structures for categorical variables. To overcome this limitati on, we introduce \textbf{WE}ighted Euclidean distance matrices \textbf{G}aussian \textbf{P}rocess (\name). \name constructs the kernel function for each categorical input by estimating the Euclidean distance matrix (EDM) among all categorical choices of this input. The EDM is represented as a linear combination of several predefined base EDMs, each scaled by a positive weight. The weights, along with other kernel hyperparameters, are inferred using a fully Bayesian framework.
We analyze the predictive performance of \name theoretically. 
Numerical experiments validate the accuracy of our GP model, and by \name, into Bayesian Optimization (BO), we achieve superior performance on both synthetic and real-world optimization problems. The code is available at: \url{https://github.com/pmy0124nus/WEGP}.

\end{abstract}

\section{INTRODUCTION}

Real-world engineering and scientific challenges often involve creating surrogate models that handle mixed-input problems (including continuous and categorical inputs) using limited training data. These models need to make accurate predictions and quantify the predictive uncertainty. For example, in material design, the goal is to find atomic structures that display specific properties like mechanical strength \citep{oune2021latent}. These structures involve both categorical variables (e.g., material type) and continuous variables (e.g., temperature and pressure).
GP models~\citep{rasmussen2003gaussian} are commonly used due to their flexibility, accurate outcome prediction, and the ability to quantify uncertainty~\citep{tuo2022uncertainty,stephenson2022measuring}. They are particularly effective as surrogate models in Bayesian Optimization (BO), a framework often applied for optimizing expensive black-box functions where data is costly or time-consuming to obtain. BO has been successfully applied to mixed-input problems, including selecting chemical compounds~\citep{hernandez2017parallel}, tuning hyperparameters for machine learning models~\citep{snoek2012practical,papenmeier2023bounce}, reinforcement learning~\citep{scannell2023mode} and conducting neural architecture searches~\citep{kandasamy2018neural,nguyen2021optimal,ru2020interpretable}.

Standard GP models are mainly designed for continuous inputs. They typically rely on a kernel function, reflecting the spatial correlation between these inputs, to quantify the similarity between continuous inputs based on some distance metric, such as the Euclidean distance. However, categorical inputs do not have a spatial structure, and the distance metric needs to be redefined.
To address this challenge, a common approach is to encode categorical variables into continuous representations~\citep{zhang2020latent,deshwal2021combining,oune2021latent}.

\begingroup
\renewcommand\thefootnote{*} 
\footnotetext{Equal contribution.}
\endgroup
\begingroup
\renewcommand\thefootnote{$\dagger$}
\footnotetext{Corresponding author: \href{mailto:haowei_wang@u.nus.edu}{haowei\_wang@u.nus.edu}}
\endgroup

We propose \textbf{WE}ighted Euclidean Distance Matrices \textbf{GP} (\name), a novel approach for capturing correlation between mixed-type inputs.
Unlike traditional methods that rely on proper encoders to capture spatial correlation, \name focuses on learning the Euclidean distance matrix (EDM) for every categorical input directly. For instance, suppose there is a categorical input $h$, taking values $h_a, h_b$ and $h_c$. \name learns a $3\times3$ distance matrix for $h$, which measures distance between every possible pair of values from $h$. When building the GP model, the distance in dimension $h$ between any two input points can then be obtained from this matrix. 
In \name, the EDM is represented as a positive linear combination of several predefined base EDMs, representing different relationships between categories by providing a distance structure. 
Base EDMs can be derived through various methods. We use two methods to generate base EDMs in this paper: the first method uses ordinal encoders to assign numerical values to categories based on their order, capturing a simple linear structure of distances among categories. The second method uses extreme direction matrices, which represent the edges of the EDM cone, allowing the representation of any EDM. By learning the weights for every base EDM, \name\ captures the importance of these diverse distance patterns, providing a flexible way to model complex relationships between categories. \name adopts a fully Bayesian inference to automatically determine the importance of every structured correlation pattern.

Our theoretical analysis demonstrates that our GP model's posterior mean converges to the underlying black-box function. We further derive its convergence rate, showing that the convergence rate depends on the smoothness of the actual underlying function and the correlation function of the GP model.
Numerical experiments validate the accuracy of our GP model.
We also integrate \name into BO and evaluate it on several synthetic and real-world optimization problems, demonstrating state-of-the-art performance on optimization problems.
Our specific contributions include:

\begin{enumerate}
    \item We develop \name for mixed-type input space. \name learns the distance pattern for each categorical input as a weighted sum of several base EDMs to improve model fitting.
    \item We propose a fully Bayesian inference for WEGP. It can automatically determine the importance of every structured correlation pattern when data is limited or sufficient.
    \item We perform a theoretical analysis for the convergence rate of \name, showing that the posterior mean of \name converges to the underlying black-box function.
    \item A comprehensive experimental evaluation on a diverse set of mixed BO datasets demonstrates the effectiveness of \name.
\end{enumerate}

\section{BACKGROUND}

\textbf{Gaussian Process.} GP~\citep{rasmussen2003gaussian} is a non-parametric Bayesian framework for modeling unknown functions, widely used in regression and classification tasks. 
A GP is defined by its mean function $\mu(x)$ and covariance function $\Sigma(x, x^{\prime})$, which determine the properties of the functions it models. Specifically, for any finite set of input points $\mathbf{x}=[x_1, \ldots, x_n]^{\top}$, the corresponding function values $\mathbf{f}=[f(x_1), \ldots, f(x_n)]^{\top}$ are assumed to follow a joint Gaussian distribution:
$$
\mathbf{f} \sim \mathcal{N}(\boldsymbol{\mu}_0, \Sigma_0)
$$
where $\mu_0$ is the mean vector, $\Sigma_0$ is the $n \times n$ covariance matrix defined by the chosen covariance function. The conditional distribution of $\mathbf{f}$ given these observations can be computed using Bayes' rule~\citep{frazier2018tutorial}. Common covariance functions include the Gaussian (squared exponential) kernel,
$$\Sigma\left(x, x^{\prime}\right)=\sigma^2 \exp \left(-\frac{\left\|x-x^{\prime}\right\|^2}{2 l^2}\right)$$

where $\sigma^2$ is the process ariance, $l$ is the length-scale. 

The hyperparameters of the covariance function are typically optimized by maximizing the marginal likelihood.

\textbf{Mixed input space.} In many practical applications, input data often consists of both continuous and categorical components, referred to as mixed input. This mixed input space can be mathematically represented as $\mathcal{Z}=\mathcal{X} \times \mathcal{H}$, where $\mathcal{X}=\left\{x_1, x_2, \ldots, x_d\right\}$ denotes the set of continuous inputs, and $\mathcal{H}=\left\{h_1, h_2, \ldots, h_c\right\}$ represents the set of categorical inputs. Consequently, each mixed input $\mathbf{z}$ can be expressed as $\mathbf{z}=(\mathbf{x}, \mathbf{h})$, combining both continuous and categorical elements. 

\textbf{Bayesian optimization.}

BO~\citep{brochu2010tutorial,shahriari2015taking,nguyen2020knowing} is an advanced extension of GP for the optimization of black-box functions that have numerous important application~\citep{korovina2020chembo,dreczkowski2024framework,dai2024batch}. 
BO leverages GP as a surrogate model to approximate the unknown objective function $f$. The iterative process of BO consists of two key steps: (1) fitting the GP given the observed data to update the posterior distribution of $f$; and (2) using the posterior to define an acquisition function $\alpha_t(\mathbf{x})$. The next sample point is determined by optimizing the acquisition function, $\mathrm{x}_t=$ $\arg \max _{x \in \mathcal{X}} \alpha_t(\mathbf{x})$. The acquisition function is computationally inexpensive to optimize, allowing BO to efficiently explore the search space, in contrast to direct optimization of the more costly objective function $f(\mathbf{x})$.

\section{RELATED WORK}
\label{sec:literature review}

\textbf{Non-GP-based surrogate models with mixed inputs.} Regression models offer effective approaches for handling mixed input types through the use of dummy variables and encoding techniques like one-hot encoding and contrasts encoding~\citep{box1992experimental, wei2008multi,hu2008optimization,naceur2006response,jansson2003using}. MiVaBO~\citep{daxberger2019mixed, baptista2018bayesian} employs a Bayesian linear regressor that captures discrete features using the BOCS~\citep{baptista2018bayesian} and continuous features through random Fourier features, incorporating pairwise interactions between them. MVRSM~\citep{bliek2021black} combines linear and ReLU units to handle mixed inputs efficiently. Some optimization models are specifically designed to handle mixed input types, making them suitable for scenarios involving both continuous and categorical variables. iDONE~\citep{bliek2021black} utilizes piece-wise linear models, offering simplicity and computational efficiency. Random forests (RFs)~\citep{breiman2001random}, employed in method SMAC~\citep{hutter2011sequential}, can naturally accommodate continuous and categorical variables. RFs are robust but can overfit easily, so the number of trees needs to be chosen carefully to balance model complexity and performance.
Another tree-based approach is the Tree Parzen Estimator (TPE)~\citep{bergstra2011algorithms} utilizes non-parametric Parzen kernel density estimators (KDE). 
By taking advantage of KDE's properties, TPE is capable of effectively managing both continuous and discrete variables~\citep{zaefferer2018surrogate}.

\textbf{GP-based Surrogate Models with Mixed Inputs.} Building metamodels for mixed input types is an emerging area for GP models, with different approaches varying in complexity. The complexity of these models depends largely on how categorical variables are handled in the kernel. More parameters enable the kernel to capture complex relationships between different categorical choices. The following methods combine the kernel for each categorical variable through multiplication. For simplicity, we analize kernel for one categorical variable with \( K \) categories.
The simplest approach is using Gower distance~\citep{halstrup2016black}, which combine Euclidean distance for continuous variables with Hamming distance for categorical variables. This method introduces a single parameter per categorical variable and is used in frameworks like COCABO~\citep{ru2020bayesian}. 
One-Hot Encoding is a more expressive approach, representing \( K \) categorical choices as a \( K \)-dimensional binary vector. This introduces \( K \) correlation parameters into the kernel, enabling the model to capture richer relationships. It is widely used in various studies~\citep{golovin2017google, garrido2020dealing, gonzalez2016batch, snoek2012practical}.
BODI ~\citep{deshwal2023bayesian} refines this by mapping One-Hot encoding into a lower-dimensional feature space, reducing the parameters to $m$ ($\leq K$).
More accurate methods like LVGP~\citep{zhang2020latent} and LMGP~\citep{oune2021latent} encode categories into an continuous space, the complexity and the accuracy of the model depends on the dimension of the continuous space.\\
All the method above still lack accuracy and considered as approximation model for categorical inputs. The accurate GP model can capture all $K(K-1)/2$ pairwise relationship between the categorical choices, for example, HH~\citep{zhou2011simple}, EHH~\citep{saves2023mixed} and UC~\citep{qian2008gaussian} has $K(K-1)/2$ parameters in the kernel. The WEGP method can be regarded as a generalization of the Unrestrictive Covariance (UC) approach. Initially proposed by \citep{qian2008gaussian}, the UC method directly estimates the $K \times K$ correlation matrix for categorical, requiring $K(K-1) / 2$ parameters to capture all pairwise relationships. To ensure the positive definiteness of the correlation matrix, the original UC method employed semidefinite programming. Subsequently, \citep{zhou2011simple} introduced the HH method, which simplified the estimation process using hypersphere decompositions, and \citep{zhang2015computer} further refined the approach by utilizing indicator variables within the Gaussian correlation function, thereby eliminating the need for semidefinite programming. WEGP also directly estimates the correlation matrix; however, it does so by employing a weighted combination of base Euclidean Distance Matrices (EDMs). Each base EDM is positive definite, and as long as the weights are non-negative, the resulting weighted matrix remains positive definite. This formulation enables the use of standard techniques such as Maximum Likelihood Estimation (MLE), Maximum A Posteriori (MAP), or fully Bayesian methods under simple non-negativity constraints. Moreover, the parameterization in WEGP offers improved interpretability, as each weight reflects the importance of a particular similarity relationship, which facilitates the use of sparsity-inducing priors in Bayesian frameworks. This advantage is particularly significant in settings with limited training data, where estimating $K(K-1) / 2$ parameters-as required by the UC method-may lead to substantial estimation errors.

\section{\name}
\label{sec:model}

\textbf{Problem statement.} We consider building a GP model for an unknown function $y(\cdot)$ over mixed inputs $\mathbf{z}$. Here, $y(\cdot)$ is the response value; $\mathbf{z}=(\mathbf{x}, \mathbf{h}) \in \mathcal{Z}$ is the design vector, where $\mathbf{x}=(x^{(1)}, x^{(2)}, \ldots, x^{(d)})$ is a vector contain values for $d$ continuous variables, $\mathbf{h}=(h^{(1)}, h^{(2)}, \ldots, h^{(c)})$ is a vector that contains the values for $c$ categorical variables, and $\mathcal{Z}$ is mixed input domain. For the $k^{th}$ categorical variable, it contains $c_k$ different categorical choices. We construct the relationship between $y$ and the mixed inputs $\mathbf{z}$ by:
\begin{equation}
    y(\mathbf{z})=\mu+G(\mathbf{z}),
\end{equation}
where $\mu$ is the constant prior mean, and $G(\mathbf{z})$ is a zero-mean GP with covariance function:
\begin{equation}
    K(\cdot, \cdot)=\sigma_0^2 R(\cdot, \cdot),
\end{equation}
where $\sigma_0^2$ is the process variance.

\subsection{Direct Euclidean Distance Matrix Estimation}
To extend GP models from continuous to categorical input spaces, we propose constructing a GP kernel for categorical inputs by directly learning the Euclidean distance matrix between categories. First, we introduce the definition of Euclidean distance matrix in mathematics:
\begin{definition}
A Euclidean distance matrix is an $n \times n$ matrix representing the spacing of a set of $n$ points in Euclidean space. For points $x_1, x_2, \ldots, x_n$ in $k$-dimensional space $\mathbb{R}^k$, the elements of their Euclidean distance matrix $D$ are given by squares of distances between them. That is
$$
\begin{aligned}
D & =(D_{i j}), \text{ where } D_{i j}  = &d_{i j}^2=\left\|x_i-x_j\right\|^2
\end{aligned}
$$
where $\|\cdot\|$ denotes the Euclidean norm on $\mathbb{R}^k$.
$$
D=\left[\begin{array}{ccccc}
0 & d_{12}^2 & \ldots & d_{1 n}^2 \\
d_{21}^2 & 0  & \ldots & d_{2 n}^2 \\

\vdots & \vdots & \ddots & \vdots \\
d_{n 1}^2 & d_{n 2}^2 & \ldots & 0
\end{array}\right].
$$
\end{definition}
In \name, the EDM for a categorical variable measures distances between every possible pair of categories. For $k^{th}$ categorical variable containing $c_k$ number of categories, the size of its EDM $D_k$ is $c_k \times c_k$. \name computes EDM $D_k$ through a positive linear combination of $m_k$ base EDMs: $D^{(1)}_k, D^{(2)}_k, \ldots, D^{(m_k)}_k$. The positive linear combination of base EDMs is still a valid EDM, which is guaranteed by the following proposition:
\begin{proposition}
Denote $m$ linearly independent $n \times n$ base EDMs as $D^{(1)}, D^{(2)}, \ldots, D^{(m)}$.
Define D as positive linear combination of the matrices $D^{(1)}, D^{(2)}, \ldots, D^{(m)}$:
$$
D=\sum_{i=1}^{m} w_i D^{(i)} \quad \text { where } \quad w_i \geq 0
$$
$D$ is also a valid Euclidean distance matrix.
\label{proposition:linearindependent}
\end{proposition}
\begin{proof}
    See Appendix~\ref{subsec:edm}.
\end{proof}
A straightforward way to compute $D_k$ is to treat each element in the matrix as a variable and estimate $D_k$ element-wisely. However, we must ensure $D_k$ is a valid distance matrix. Thus, Schoenberg's theorem~\citep{schoenberg1935remarks} needs to be satisfied, which means estimating the element in EDM directly involves complex positive definite programming with several constraints.
In contrast, our method inherently guarantees a valid EDM by using the positive linear combination of base EDMs, thereby simplifying the optimization process in building the matrix.

Moreover, computing the EDM through a linear combination of predefined EDMs allows us to fully leverage existing structured distance information. Specifically, by optimizing the weighting coefficients, we can identify which distance information within the predefined EDMs is more crucial for constructing the kernel function. A base EDM represents a type of relationship between categories by providing a simplified distance structure. These base EDMs can be derived through various methods. In this work, we construct base EDMs using two methods.

\textbf{Construct base EDMs by ordinal encoders.} The ordinal encoder assigns numerical values to each category based on their order and then calculates the pairwise distances between categories. 
The base EDM is constructed accordingly. 
Here is an example to illustrate the construction of a base EDM using an ordinal encoder: 
\begin{example}
    Consider a categorical variable $h^{(1)}$ with three categories: $h_1^{(1)}=A$, $h_2^{(1)}=B$, and $h_3^{(1)}=C$. An ordinal encoder represents a mapping from the categories to ordinal values. For example, an ordinal encoder encodes $(A, B, C)$ as $(1,2,3)$. The base EDM it generates is: 
$$
D^{(1)}_1=\left(\begin{array}{lll}
0 & 1 & 4 \\
1 & 0 & 1 \\
4 & 1 & 0
\end{array}\right)
$$
This base EDM contains simple structured distance information that the distance between $A, B$, and $B, C$ is the same. The distance between $A, C$ is the largest among three pairwise distances, twice the distance between $A, B$ and $B, C$.
\end{example}

By permuting the encoding, we can generate different base EDMs corresponding to different distance patterns of the relative positions among the categories. We apply Algorithm \ref{alg:construct_basis_edms_multiple} to construct $m_k$ linear independent base EDMs for categorical input $h^{(k)}$ based on different ordinal encoders.

Using ordinal encoders to generate base EDMs allows us to preserve the inherent order of categories in their numerical representation. When the sample size is small, it becomes challenging to accurately determine the exact numerical distances between these categories due to the limited amount of data. Instead of relying on potentially unreliable distance measurements, we leverage the relative position information, which remains identifiable even with sparse data, making it a more practical approach in such scenarios.

\begin{algorithm}
\caption{Construct base EDMs with ordinal encoders.}
\label{alg:construct_basis_edms_multiple}
\textbf{Input:} $c$ categorical variables and \( m_k \) base EDMs for the $k^{th}$ categorical variable.
\begin{algorithmic}[1]

\For{each categorical variable \( h^{(k)}\) where \( k = 1, 2, \ldots, c \)}
    \State Initialize an empty set \( S_k \gets \emptyset \)
    \While{$|S_k| < m_k$}
        \State Randomly select a permutation for ordinal encoding
        \State Perform ordinal encoding based on the selected permutations
        \State Calculate EDM \( D_k^{(i)} \)
        \If{ \( D_k^{(i)} \) is linearly independent from all matrices in \( S_k \) }
            \State Add \( D_k^{(i)} \) to \( S_k \)
        \EndIf
    \EndWhile
    \State Store \( S_k \) for \( h^{(k)} \)
\EndFor
\State \textbf{Return} \( \{ S_1, S_2, \ldots, S_c \} \)
\end{algorithmic}
\end{algorithm}

\textbf{Construct base EDMs by extreme directions of EDM cone.} The positive linear combination of base EDMs generated by ordinal encoders can only span a subspace of the EDM cone. 
To accurately represent \textit{any} Euclidean Distance Matrix (EDM), we propose a method that exploits the structural properties of the EDM cone. This approach is based on the mathematical nature of EDMs as elements within a convex cone. In what follows, we first formalize the concept of the EDM cone, explore its structure, and demonstrate how the extreme directions of this cone can serve as base EDMs for constructing any EDM.

\begin{definition}[\citep{boyd2004convex}]
In the space of $n \times n$ symmetric matrices, the set of all $n \times n$ Euclidean Distance Matrices (EDMs) forms a unique, immutable, pointed, closed convex cone called the EDM cone, denoted as $\mathcal{E}_n$. Specifically,
\begin{align*}
    \mathcal{E}_n = \big\{ D \in \mathbb{R}^{n \times n} \mid &\ D = [d_{ij}], \ d_{ij} = \|x_i - x_j\|_2, \\
    \ d_{ii} &= 0 \text{ for all } i, j \big\}.
\end{align*}
The dimension of the cone $\mathcal{E}_n$ is ${n(n-1)}/{2}$.
\label{cone}
\end{definition}

\begin{remark}
The dimensionality, ${n(n-1)}/{2}$, reflects the degrees of freedom in the pairwise distances for $n$ points in Euclidean space.
\label{dim_remark}
\end{remark}
The definition \ref{cone} provides a formal foundation for analyzing EDMs as elements of a specific convex structure.

\begin{lemma}[Carathéodory's theorem] Let X be a nonempty subset of $\mathbb{R}^n$.
Every nonzero vector of cone(X) can be represented as a positive combination of $n$ linearly independent vectors from X.
\label{lemma:cara}
\end{lemma}
 
Carathéodory’s theorem~\citep{deza1997geometry,hiriart2004fundamentals} guarantees that any vector within a convex cone, such as an EDM, can be expressed as a positive linear combination of a finite number of basis vectors. This result is critical for establishing that any EDM can be constructed using a limited set of linearly independent matrices.
\begin{definition}[Extreme Directions of the EDM Cone]
An extreme direction of the EDM cone corresponds to the case where the affine dimension $r=1$. For any cardinality $N \geq 2$, each nonzero vector $z$ in $\mathcal{N}(\mathbf{1}^{\mathrm{T}})$, where $\mathcal{N}$ denotes the null space, can be used to define an extreme direction $\Gamma \in \mathbb{E D M}^N$ as follows:
\begin{align}
\Gamma & \triangleq (z \circ z) \mathbf{1}^{\mathrm{T}} + \mathbf{1} (z \circ z)^{\mathrm{T}} - 2 z z^{\mathrm{T}} \in \mathbb{E D M}^N 
\label{form:extreme direction}
\end{align}
where $\Gamma$ represents a ray in an isomorphic subspace $\mathbb{R}^{N(N-1)/2}$, corresponding to a one-dimensional face of the EDM cone.
\label{extreme}
\end{definition}
Extreme directions, constructed by Eq.~\ref{form:extreme direction}, constitute the fundamental elements that delineate the minimal boundaries of the EDM cone. These directions correspond to rays extending along specific axes of the cone and are essential for generating any element within the cone through positive linear combination. Consequently, by selecting extreme directions as the base EMDs, it is possible to compute any EDM. To substantiate this claim, we combine Carathéodory’s theorem with the characterization of extreme directions, leading to the proposition below.

\begin{proposition}
    Any vector EDM matrix, in the cone $\mathcal{E}_n$ can be constructed using positive linear combination of ${n(n-1)}/{2}$ linearly independent extreme directions.
\label{prop:extreme}
\end{proposition}
\begin{proof}
    See Appendix~\ref{subsec:any_edm}.
\end{proof}
This proposition highlights that any possible EDM can be generated through positive linear combination of a set of ${n(n-1)}/{2}$ linearly independent extreme direction matrices. The ability to construct any EDM in this manner ensures that, with sufficient data, the coefficients of these combinations can be estimated accurately, thereby allowing for an effective approximation of the true EDM. Consequently, this method enables a comprehensive representation of the relationships between points in the underlying space, fulfilling the objectives of our approach. 

Beyond those two methods mentioned above, other techniques can also be used to create base EDMs, each offering a different perspective on the relationships between categories. This flexibility allows \name\ to adapt to diverse types of data and capture the complex correlations within it.

\textbf{Kernel construction by EDM.} Consider two inputs $\mathbf{z_p}=[\mathbf{x_p}, \mathbf{h_p}]$ and $\mathbf{z_q}=[\mathbf{x_q}, \mathbf{h_q}]$. The categorical kernel $R_k$ for $\mathbf{h}^{(k)}$ is defined based on $D_k$:
\begin{align}
R_k\left(\mathbf{h}_p^{(k)}, \mathbf{h}_q^{(k)} \mid \boldsymbol{w}_k\right) &= \exp \left(-D_{k,p q}\right) \\
&=\exp \left(-\sum_{i=1}^{m_k} w_{k}^{(i)} D_{k,}^{(i)}{ }_{p q}\right),    
\end{align}
where $D_{k,}^{(i)}{}_{p q}$ denotes the distance between categories $\mathbf{h}_p^{(k)}$ and $\mathbf{h}_q^{(k)}$ in the $i$-th base EDM, and ${w}_k^{(i)}$ are the corresponding weights.
For $d$ continuous variables, the kernel is given by:
$$
R\left(\mathbf{x}_p, \mathbf{x}_q \mid \boldsymbol{\theta}\right)=\exp \left(-\sum_{j=1}^d \theta_j\left\|x^{(j)}_{p}-x^{(j)}_{q}\right\|^2\right).
$$
The overall kernel for mixed inputs combines the continuous and categorical kernels multiplicatively:
\begin{align}
K&\left(\mathbf{z}_p, \mathbf{z}_q \mid \sigma, \boldsymbol{\theta},\left\{\boldsymbol{w}_k\right\}\right) \\
&=\sigma_0^2 R\left(\mathbf{x}_p, \mathbf{x}_q \mid \boldsymbol{\theta}\right) \prod_{k=1}^c R_k\left(\mathbf{h}_p^{(k)}, \mathbf{h}_q^{(k)} \mid \boldsymbol{w}_k\right).
\label{eq:kernel_com}
\end{align}

\subsection{Kernel Hyperparameters Estimation}
\label{section: model estimation}
 
From the Eq.~\ref{eq:kernel_com}, the kernel hyperparameters we need to estimate are $\sigma^2,\{\theta_i\}, \tau,\{w_k^{(i)}\}$. Here, $w_k^{(i)}$ are the weights of the linear combination of different basis EDM, \(\theta_i\) is the inverse squared length scale, $\tau$ is a global shrinkage parameter, and \(\sigma_0\) is the process variance. We adopt a full Bayesian inference~\citep{frazier2018tutorial} to estimate these parameters. It allows for a more robust estimation by integrating over the uncertainty in the kernel hyperparameters.

\textbf{Hierarchical GP.} To mitigate model overfitting, we aim to achieve some sparsity of the weights $w_k^{(i)}$. To accomplish this, we employ a hierarchical GP model, a similar approach was applied in \citep{eriksson2021high} to address the high-dimensional challenge in BO.
The hierarchical structure achieves the sparsity in weights through properly chosen priors. The joint distribution of the model parameters $\sigma^2,\{\theta_i\}, \tau,\{w_k^{(i)}\}$ is expressed as:
\begin{equation*}
[\sigma^2,\{\theta_i\}, \tau,\{w_k^{(i)} \}]=[ \{w_k^{(i)} \} \mid \tau ] \times[\tau] \times [ \{\theta_i \} ] \times [\sigma^2 ]
\end{equation*}
Here the weight is governed by a global shrinkage parameter $\tau$, which also has a prior distribution. Specifically, the priors for hyper-parameters are as follows:
\begin{align*}
\text{[kernel variance]} \quad & \sigma^2 \sim \mathcal{LN}(0, 10^2) \\
\text{[length scales]} \quad & \theta_i \sim Uniform(0, 1)  \\
\text{[global shrinkage]} \quad & \tau \sim \mathcal{HC}(\alpha) \\
\text{[coefficients]} \quad & w_{k}^{(i)} \sim \mathcal{HC}(\tau) \quad \\ &\text{for } i = 1, \ldots,m_k; \\& k = 1, \ldots, c 
\end{align*}
where $\mathcal{LN}$ denotes the log-Normal distribution and $\mathcal{HC}(\alpha)$ denotes the half-Cauchy distribution, i.e. $p(\tau|\alpha) \propto (\alpha^2 + \tau^2)^{-1}(\tau > 0)$, and $p(w_{k}^{(i)}|\tau) \propto (\tau^2 + {w_{k}^{(i)}}^2)^{-1}(w_{k}^{(i)} > 0)$. $\alpha$ is a hyper-parameter that controls the level of shrinkage (default is $\alpha = 0.1$)~\citep{eriksson2021high}. The prior on the kernel variance $\sigma_k^2$ is weak to allow flexibility in the model.

The half-Cauchy priors for both the global shrinkage parameter $\tau$ and the coefficients $w_{k}^{(i)}$ encourages most coefficients to be near zero, effectively reducing the model complexity by focusing on the most relevant features. This approach aligns with automatic relevance determination \citep{mackay1994automatic}, which aims to identify and focus on the most relevant weights in the data. Moreover, the half-Cauchy priors also have heavy tails, meaning that if the observed data provide strong enough evidence with larger data set, the model “turning on” more weights. Consequently, the hierarchical GP model can adapt the level of sparsity in response to the input size, maintaining both interpretability and efficiency. 
 We implement the Bayesian inference using the No-U-Turn Sampler (NUTS) \citep{hoffman2014no}, an adaptive variant of Hamiltonian Monte Carlo for \name. The fully Bayesian framework enhances our model by enabling dynamic adjustment of the global shrinkage parameter through its distribution, rather than fixing it. Therefore, our approach allows the model to learn weighted base components that are simple yet informative when data is limited, and to approximate more complex structures as the dataset grows. This ensures that the optimization process remains effective and accurate across different data scales.

\subsection{Theoretical Analysis}
In this subsection, we aim to validate our proposed \name from a theoretical perspective. Specifically, we want to investigate whether the predictive mean will converge to the true objective function $f$, and more importantly, how fast it converges. The consistency and convergence rate are strong supporting evidence that our \name is valid. 
\begin{theorem}
\label{theorem: convergence}

For a Matérn kernel $K_\nu$ with smoothness $\nu>0$, let $H_\nu(\mathcal{Z})$ denote the RKHS of $K_\nu$ on $\mathcal{Z}.$
Assume that $f \in \mathcal{H}_\nu(\mathcal{Z})$.
Suppose that the design in $\mathcal{Z}$ has fill distance $h$. Then there exist constants $C$ (independent of $h$) such that for all $\mathbf{z} \in \mathcal{Z}$
\begin{equation}
  \left|f(\mathbf{z})-\mu_n(\mathbf{z} ; \nu)\right|\leq C\|f\|_{\mathcal{H}_\nu(\mathcal{Z})} h^{\nu \wedge 1}\left(\log \frac{1}{h}\right)^\beta .
\end{equation}
where $\beta \geq 0$ depends on $\nu$ (with no logarithmic correction when $\nu>1$ ).
\label{main theorem}
\end{theorem}
\begin{proof}
    See Appendix~\ref{sec:converge}.
\end{proof}
This theorem proves that the posterior mean function has a convergence rate of $\mathcal{O}( h^{\nu \wedge 1}\left(\log \frac{1}{h}\right)^\beta)$, and the proof of the convergence rate inherently demonstrates the convergence. It guarantees the reliability and accuracy of the predictions made by the WEGP model.

\begin{figure*}[h]
    \centering
    \begin{subfigure}[b]{0.245\textwidth}
        \centering
    \includegraphics[width=\textwidth]{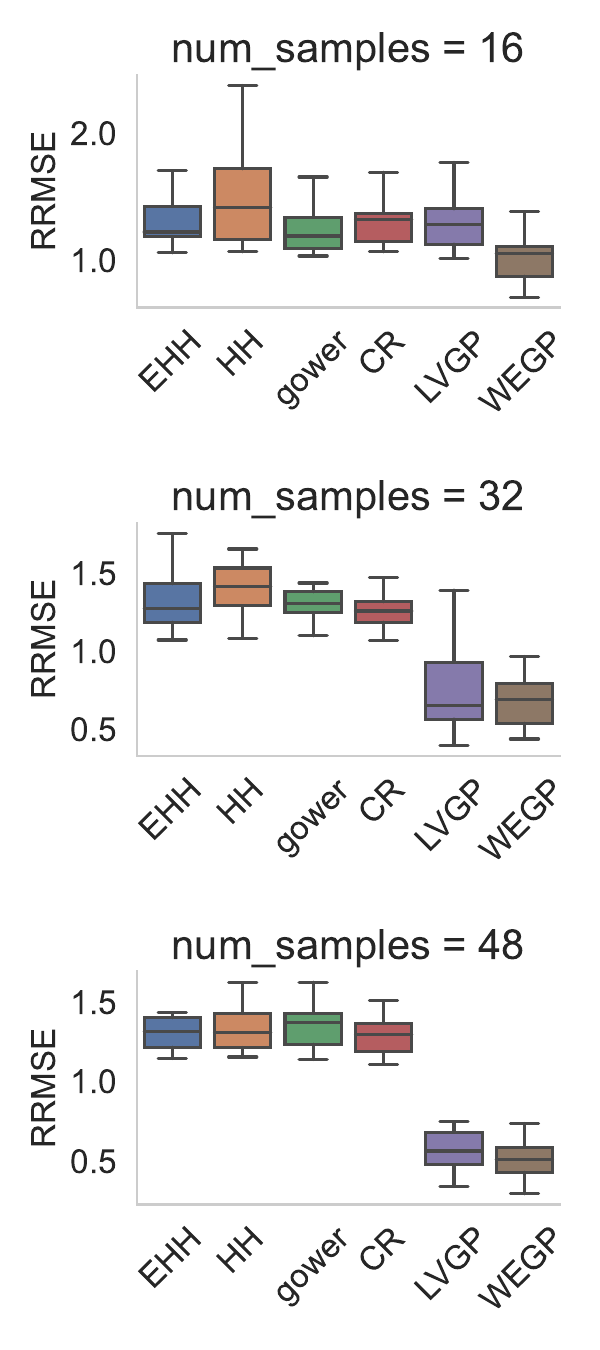}
        \caption{Beam bending}
    \end{subfigure}
    \begin{subfigure}[b]{0.245\textwidth}
        \centering
        \includegraphics[width=\textwidth]{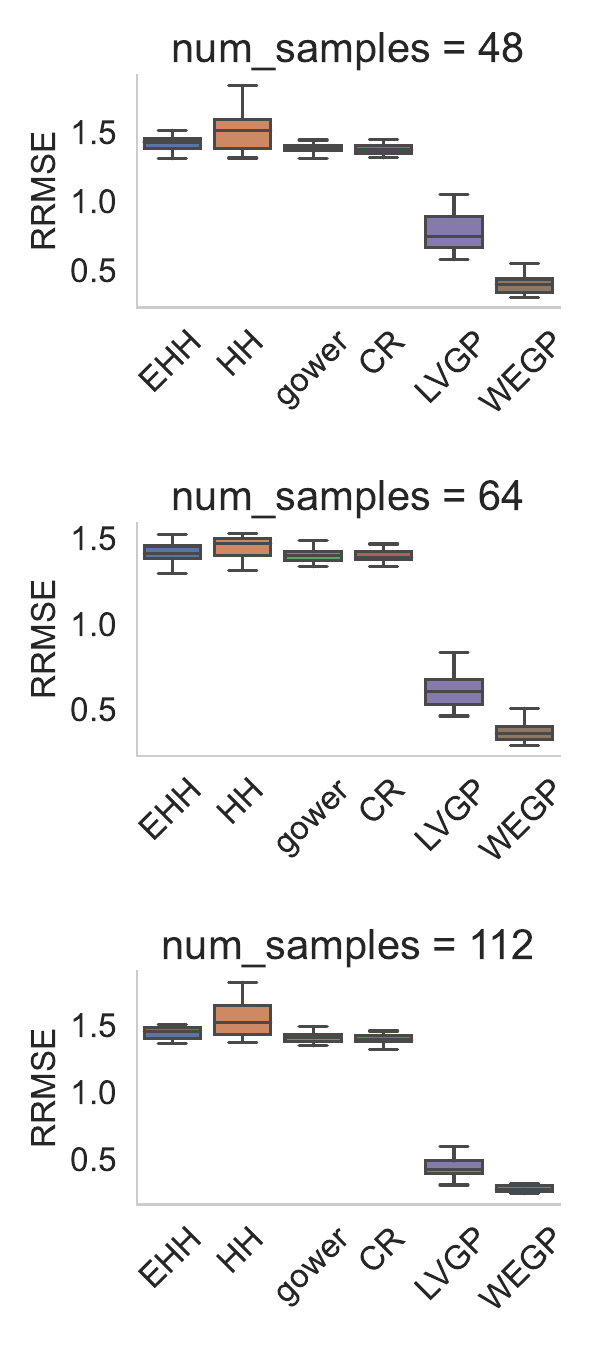}
        \caption{Piston}
    \end{subfigure}
    \begin{subfigure}[b]{0.245\textwidth}
        \centering
        \includegraphics[width=\textwidth]{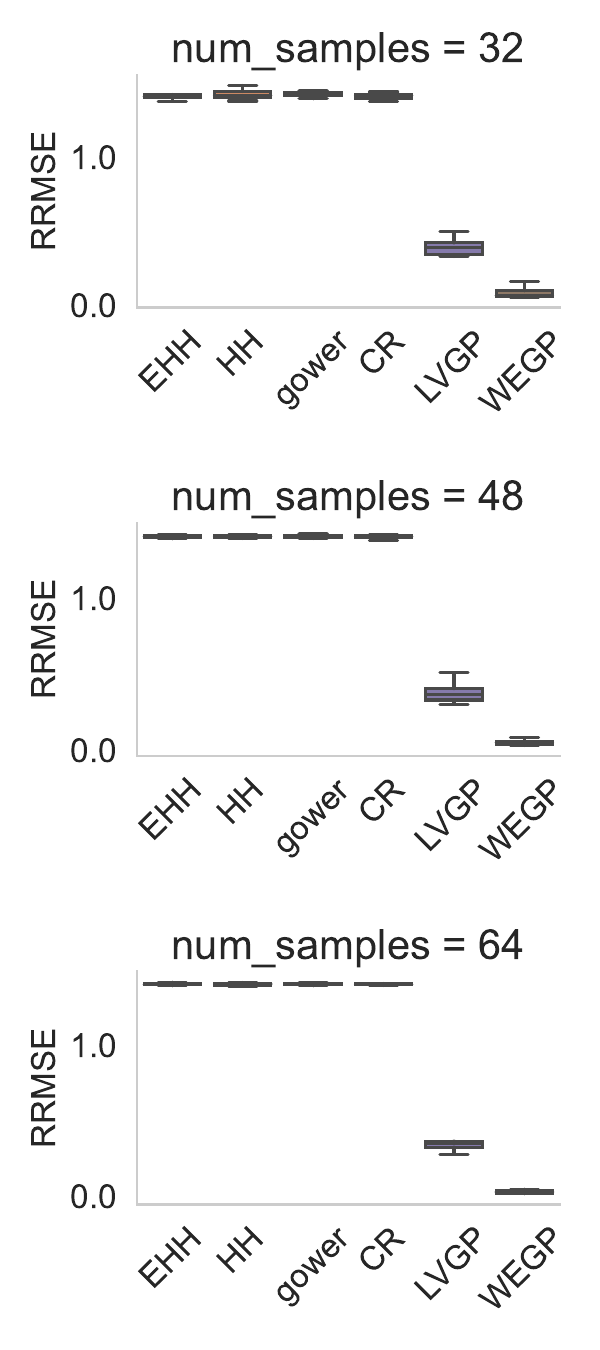}
        \caption{Borehole}
    \end{subfigure}
    \begin{subfigure}[b]{0.245\textwidth}
        \centering
        \includegraphics[width=\textwidth]{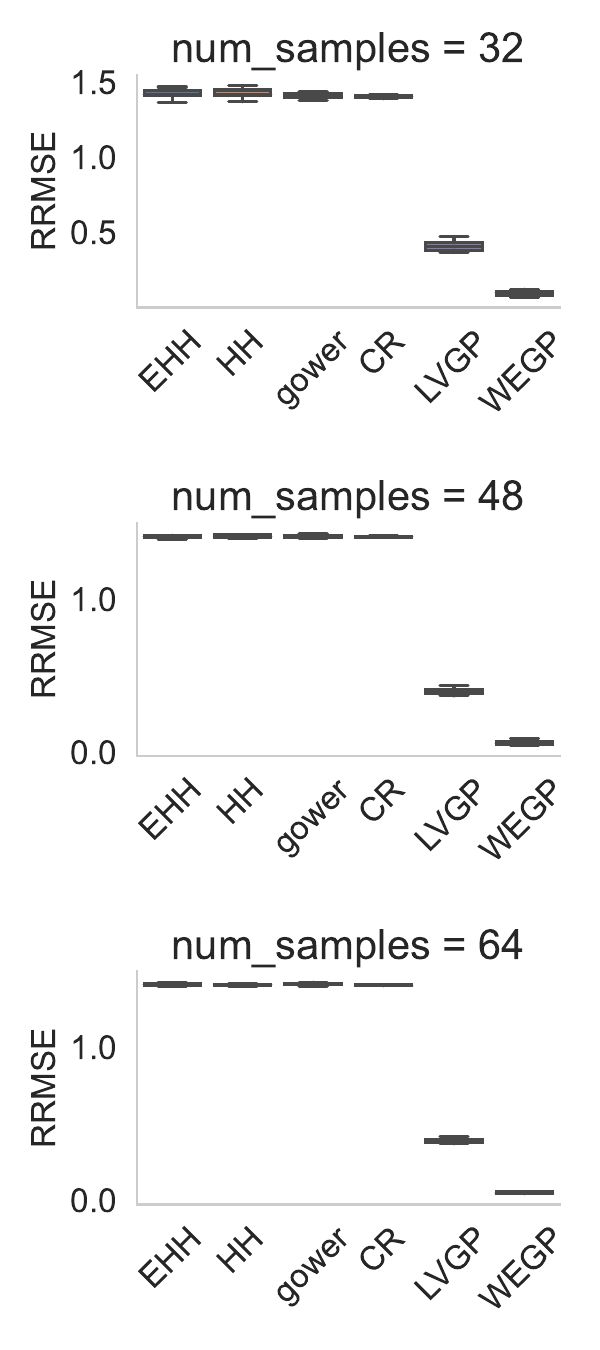}
        \caption{Otl}
    \end{subfigure}
    \caption{WEGP model accuracy comparison under three different sizes of training data}
    \label{fig:image_matrix}
\end{figure*}
\section{WEBO}

We apply \name model in Bayesian Optimization (BO) to extend its usage in black-box optimization problems with mixed input, name it as WEBO.
We use expected improvement (EI) as our acquisition function due to computational efficiency, and well-accepted empirical performance. The EI function depends on the kernel hyperparameters $\phi=\{\sigma, \mathbf{\theta},w_k\}$ through the GP model. Since we use NUTS to sample hyperparameters from the posterior distribution, the expected improvement is defined by averaging EI over $L$ posterior samples:
\begin{equation}
\operatorname{EI}\left(\mathbf{z} \mid {y}_{\min },\left\{\phi_{\ell}\right\}\right) \equiv \frac{1}{\mathrm{~L}} \sum_{\ell=1}^{\mathrm{L}} \operatorname{EI}\left(\mathbf{z} \mid {y}_{\min }, \phi_{\ell}\right)
\label{eq:EI}
\end{equation}
Since Eq.~\ref{eq:EI} is non-differentiable in a mixed input space, we optimize it by selecting the point with the highest acquisition function value from 500 randomly generated design data as the query point. Algorithm \ref{alg:webo} provides a complete outline of the WEBO algorithm.
\begin{algorithm}[h]
\caption{WEBO Algorithm}
\label{alg:webo} 
\begin{algorithmic}[1]
\State \textbf{Input:} A black-box function $f$, observation data $D_0$, maximum number of iterations $T$
\State \textbf{Output:} minimum objective function value $\left( \mathbf{z}_{min},y_{min} \right)$, where $\mathbf{z}_{min}=(\mathbf{x}_{min},\mathbf{h}_{min})$
\For{$t = 1, \dots, T$}
    \State Let $y_{\min }^t=\min _{s<t} y_s$
    \State Fit \name to $\mathcal{D}_{t-1}$ using NUTS to obtain hyper-parameter samples $\left\{\phi^t_\ell\right\}$.    
    \State Optimize EI to obtain $\mathbf{z}_t$:
    $$\mathbf{z}_t=\operatorname{argmax}_{\mathbf{x,h}} \mathrm{EI}(\mathbf{x,h} \mid \mathrm{y}_{\min }^{\mathrm{t}},\phi_\ell^{\mathrm{t}})$$
    \State Query at $\mathbf{z}_t = (\mathbf{x}_t, \mathbf{h}_t)$ to obtain $f_t(\mathbf{z}_t)$, $D_t \leftarrow D_{t-1} \cup \left(\mathbf{z}_t,f_t\left(\mathbf{z}_t\right)\right)$
\EndFor
\State \textbf{return} $(\mathbf{z}_{\min }, y_{\min })$ where $(\mathbf{o}_{\min }, y_{\min }) \equiv(\mathbf{x}_{t_{\min }}, y_{t_{\min }})$ and $t_{\min }=\operatorname{argmin}_t y_t$.
\end{algorithmic}
\end{algorithm}

\section{EXPERIMENTS}
\subsection{WEGP}
In this section, we evaluate the predictive accuracy of the WEGP model. We construct the basis EDM by the ordinal encoder, and set $m_k=c_k(c_k-1)/2$ base EDMs for $k^{th}$ categorical variable with $c_k$ categories to capture all possible pairwise relationships among the $c_k$ categorical choices. For comparative analysis, we use open-source implementations for baseline models LVGP~\citep{zhang2020latent}, Gower distance~\citep{halstrup2016black}, Continuous Relaxation (CR) \citep{golovin2017google, garrido2020dealing}, Hypersphere Decomposition (HH) \citep{zhou2011simple}, and Enhanced Hypersphere Decomposition (EHH) \citep{saves2023mixed}. 
The prediction quality of the methods is quantified as the relative root-mean-squared error (RRMSE) of their predictions over $N$ test points: $
\operatorname{RRMSE}=\sqrt{\frac{\sum_{i=1}^N(y_i-\widehat{y}_i)^2}{\sum_{i=1}^N(y_i-\bar{y})^2}},
$
where $y_i$ and $\widehat{y}_i$ denote the true and predicted values respectively for the $i^{\text {th }}$ test sample, and $\bar{y}$ is the average across the $N$ true test observations. For each problem, we utilize three different sizes of training data and perform 15 independent macro-replications to ensure the robustness of the results. The codebase is built on top of the GPyTorch~\citep{gardner2018gpytorch}.  

\textbf{Test functions.} We evaluate the performance of the \name and LVGP models using four engineering benchmarks that are commonly employed for surrogate modeling with mixed inputs. Detailed descriptions and formulations of each model are provided in the Appendix~\ref{subsec:model test}.

\begin{figure*}[t]
    \centering
    \includegraphics[width=\linewidth]{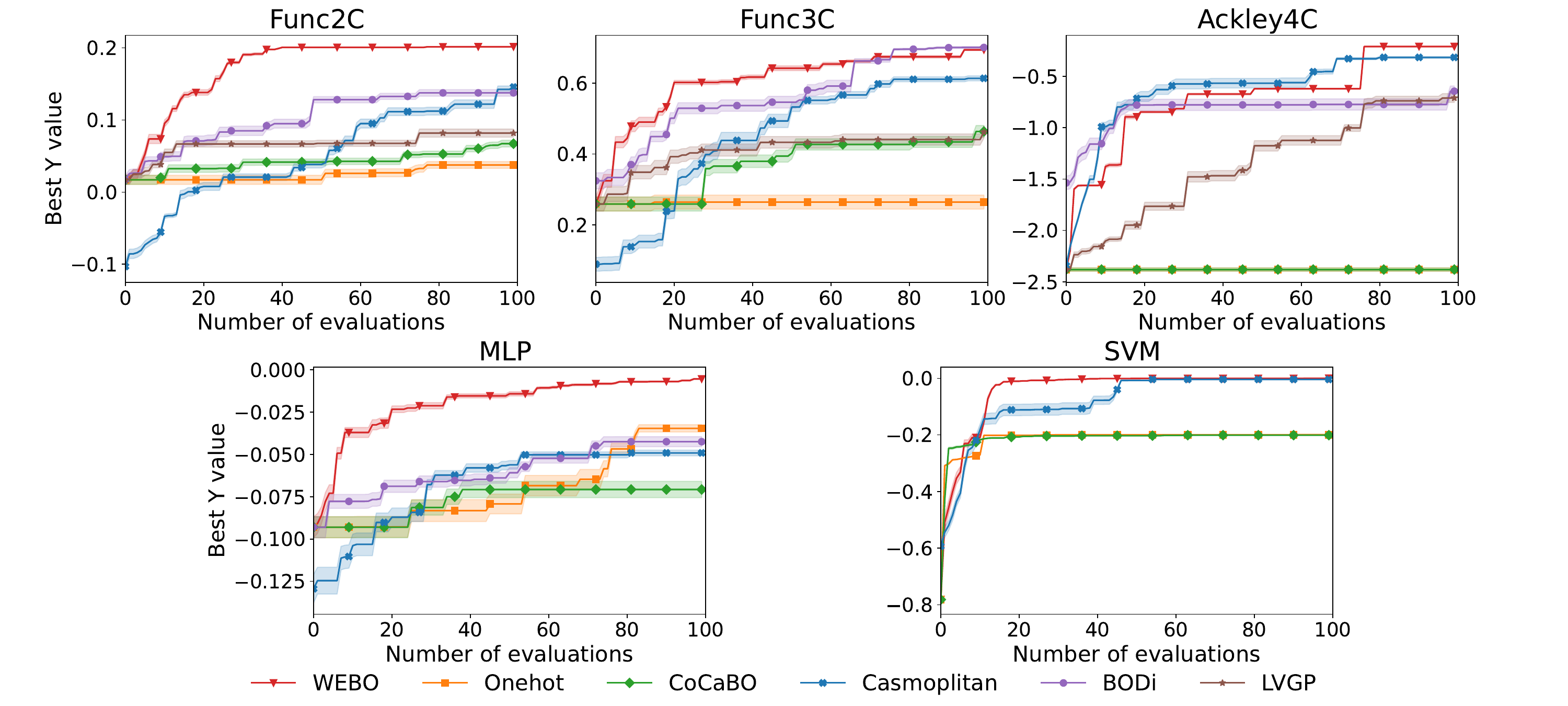}
    \caption{Results on mixed input optimization problems. Lines and shaded area denote mean and standard error.}
    \label{fig:webo}
\end{figure*}
\textbf{Model performance.} \name consistently outperforms LVGP on four test functions, especially when there are fewer training observations. This superior performance can be attributed to \name's ability to utilize the distance information provided by base EDMs, whereas LVGP has to estimate the distance from the ground up. As the number of training observations increases, the prediction errors for both \name and LVGP decrease. However, it is important to note that the presence of noise tends to increase the overall prediction errors for both methods. Despite this, \name maintains its edge over LVGP, demonstrating robust performance.

\subsection{WEBO}

We evaluate WEBO on various optimization problems for mixed inputs with continuous and categorical inputs. According to the analysis in \citep{dreczkowski2024framework}, we compare against several competitive baselines, including BODi~\citep{deshwal2303bayesian},  CASMOPOLITAN~\citep{wan2021think} , CoCaBO~\citep{ru2020bayesian}, and BO based on onehot encoding~\citep{golovin2017google, garrido2020dealing}.
LVGP~ is also considered in synthetic problems by adding the acquisition function the same as WEGP.

\textbf{Test functions.} Most of the works’ experimental sections (such as \citep{ru2020bayesian}, \citep{deshwal2023bayesian}, \citep{wan2021think} and \citep{papenmeier2023bounce}) focus on problems involving binary variables. Since only one pairwise relationship in binary variables, the complexity (i.e., the number of parameters in the kernel) of the GP model does not significantly impact the output, less complex GP model, which they use, it acceptable.
In contrast, real-world problems often involve categorical variables with multiple categories. In such cases, simple GP models are unable to accurately capture all pairwise relationships between categories, which can adversely affect model accuracy. Our model, however, is designed to effectively capture relationships between categories even in multi-class problems, resulting in improved accuracy. We tested all these methods on synthetic problems and real-world problems with multiple categories. Detailed information is provided in Appendix~\ref{subsec:optimization_test}, here is a brief sketch:
\begin{itemize}
    \item \textit{Func2C}, with $c=2$ and $d=2$, and \textit{Func3C} with $c=3$, $d=3$, respectively.
    \item \textit{Ackley4C}, with $c=4$ and $d=3$. Each categorical variable contain 3 categories.
    \item \textit{MLP}, with $c=3$ and $d=3$, we tunes 3 categorical and 3 continuous hyperparamters for MLP. Each categorical variable contains 3 choices. MSE is evaluated to measure the performance.
    \item \textit{SVM}, with $c=1$ and $d=2$, the categorical variable has 4 choices. MSE is evaluated to measure the performance.
\end{itemize}

\textbf{Model performance.}
We compared our proposed method with the five aforementioned approaches, excluding the LVGP method from the hyperparameter tuning comparison due to its prohibitively long computation time. And BODi is excluded for SVM hyperparameter tuning as it is designed for high-dimensional binary problems and does not support tasks with a categorical variable having 4 categories. By conducting eight independent replications of each experiment, we ensured statistical reliability. The results indicate that our method outperforms the alternatives by converging more rapidly to the optimal solution on test functions, achieving a lower mean squared error (MSE) in the multilayer perceptron (MLP) model on real-world problems, and exhibiting smaller confidence intervals. These smaller shaded areas demonstrate that our experimental results are more consistent and robust. Overall, these findings underscore the superior performance and reliability of our approach in both theoretical and practical applications.
\section{DISCUSSION}
\label{sec:conclusion}

We introduce a novel approach, \name, for mixed input. \name focuses on capturing structured distance information between categorical inputs through estimating EDM from a series of base distance matrices. This approach has demonstrated superior performance, particularly in data-limited scenarios, by optimizing the model's ability to learn meaningful relationships even with sparse data. While \name has shown its effectiveness across a diverse set of problems, several questions and potential areas for improvement remain.

The limitation of our current work is that, while \name is designed to be highly effective in data-sparse environments, its performance in scenarios with extremely high-dimensional categorical spaces requires further exploration. As the number of categories increases, the complexity of accurately estimating relative similarities also increases, potentially affecting the model’s scalability. Future work could explore adaptive mechanisms that dynamically adjust the dimensionality of the embedding space based on the complexity of the input space.
\section*{Acknowledgments}
This work was supported in part by the National
Natural Science Foundation of China under Grant 72101106;
in part by the Shenzhen Science and Technology Program under Grant RCBS20210609103119020. S. H. Ng’s work was partially supported by the Singapore Ministry of Education (MOE) Academic Research Fund (AcRF) [R-266-149--114].
\bibliographystyle{abbrvnat}
\bibliography{reference}

\begin{thebibliography}{56}
\providecommand{\natexlab}[1]{#1}
\providecommand{\url}[1]{\texttt{#1}}
\expandafter\ifx\csname urlstyle\endcsname\relax
  \providecommand{\doi}[1]{doi: #1}\else
  \providecommand{\doi}{doi: \begingroup \urlstyle{rm}\Url}\fi

\bibitem[Balandat et~al.(2020)Balandat, Karrer, Jiang, Daulton, Letham, Wilson,
  and Bakshy]{balandat2020botorch}
M.~Balandat, B.~Karrer, D.~R. Jiang, S.~Daulton, B.~Letham, A.~G. Wilson, and
  E.~Bakshy.
\newblock Botorch: {A} framework for efficient monte-carlo {B}ayesian
  optimization.
\newblock In \emph{Advances in Neural Information Processing Systems 33: Annual
  Conference on Neural Information Processing Systems 2020, NeurIPS 2020,
  December 6-12, 2020, virtual}, 2020.

\bibitem[Baptista and Poloczek(2018)]{baptista2018bayesian}
R.~Baptista and M.~Poloczek.
\newblock {B}ayesian optimization of combinatorial structures.
\newblock In \emph{Proc. of ICML}, volume~80 of \emph{Proceedings of Machine
  Learning Research}, pages 471--480. {PMLR}, 2018.

\bibitem[Baumert et~al.(1962)Baumert, Golomb, and
  Hall~Jr]{baumert1962discovery}
L.~Baumert, S.~W. Golomb, and M.~Hall~Jr.
\newblock Discovery of an hadamard matrix of order 92.
\newblock 1962.

\bibitem[Bergstra et~al.(2011)Bergstra, Bardenet, Bengio, and K{\'{e}}gl]{tpe}
J.~Bergstra, R.~Bardenet, Y.~Bengio, and B.~K{\'{e}}gl.
\newblock Algorithms for hyper-parameter optimization.
\newblock In \emph{Advances in Neural Information Processing Systems 24: 25th
  Annual Conference on Neural Information Processing Systems 2011. Proceedings
  of a meeting held 12-14 December 2011, Granada, Spain}, pages 2546--2554,
  2011.

\bibitem[Bernasconi(1987)]{LABS_statistical_physics}
J.~Bernasconi.
\newblock Low autocorrelation binary sequences:statistical mechanics and
  configuration space analysis.
\newblock \emph{Journal de Physique}, 48\penalty0 (4):\penalty0 559--567, 1987.

\bibitem[Biere et~al.(2009)Biere, Heule, and van Maaren]{biere2009handbook}
A.~Biere, M.~Heule, and H.~van Maaren.
\newblock \emph{Handbook of satisfiability}, volume 185.
\newblock IOS press, 2009.

\bibitem[Clark et~al.(2021)Clark, Connors, Stevenson, Hromada, Hamilton,
  Amador-Noguez, and Venturelli]{clark2021design}
R.~L. Clark, B.~M. Connors, D.~M. Stevenson, S.~E. Hromada, J.~J. Hamilton,
  D.~Amador-Noguez, and O.~S. Venturelli.
\newblock Design of synthetic human gut microbiome assembly and butyrate
  production.
\newblock \emph{Nature communications}, 12\penalty0 (1):\penalty0 1--16, 2021.

\bibitem[Daulton et~al.(2022)Daulton, Wan, Eriksson, Balandat, Osborne, and
  Bakshy]{daulton2022pr}
S.~Daulton, X.~Wan, D.~Eriksson, M.~Balandat, M.~A. Osborne, and E.~Bakshy.
\newblock {B}ayesian optimization over discrete and mixed spaces via
  probabilistic reparameterization.
\newblock \emph{arXiv preprint arXiv:2210.10199}, 2022.

\bibitem[Deshwal and Doppa(2021)]{LADDER}
A.~Deshwal and J.~R. Doppa.
\newblock Combining latent space and structured kernels for {B}ayesian
  optimization over combinatorial spaces.
\newblock In \emph{Advances in Neural Information Processing Systems
  (NeurIPS)}, pages 8185--8200, 2021.

\bibitem[Deshwal et~al.(2020)Deshwal, Belakaria, Doppa, and Fern]{l2s_disco}
A.~Deshwal, S.~Belakaria, J.~R. Doppa, and A.~Fern.
\newblock Optimizing discrete spaces via expensive evaluations: {A} learning to
  search framework.
\newblock In \emph{Proceedings of the AAAI Conference on Artificial
  Intelligence}, volume~34, pages 3773--3780, 2020.

\bibitem[Deshwal et~al.(2021{\natexlab{a}})Deshwal, Belakaria, and
  Doppa]{MerCBO}
A.~Deshwal, S.~Belakaria, and J.~R. Doppa.
\newblock Mercer features for efficient combinatorial {B}ayesian optimization.
\newblock In \emph{AAAI conference on Artificial Intelligence},
  2021{\natexlab{a}}.

\bibitem[Deshwal et~al.(2021{\natexlab{b}})Deshwal, Belakaria, and
  Doppa]{deshwal2021bayesian}
A.~Deshwal, S.~Belakaria, and J.~R. Doppa.
\newblock {B}ayesian optimization over hybrid spaces.
\newblock In \emph{Proc. of ICML}, volume 139 of \emph{Proceedings of Machine
  Learning Research}, pages 2632--2643. {PMLR}, 2021{\natexlab{b}}.

\bibitem[Djokovi{\'c} et~al.(2014)Djokovi{\'c}, Golubitsky, and
  Kotsireas]{djokovic2014some}
D.~Z. Djokovi{\'c}, O.~Golubitsky, and I.~S. Kotsireas.
\newblock Some new orders of hadamard and skew-hadamard matrices.
\newblock \emph{Journal of combinatorial designs}, 22\penalty0 (6):\penalty0
  270--277, 2014.

\bibitem[Doppa(2021)]{IJCAI-2021}
J.~R. Doppa.
\newblock Adaptive experimental design for optimizing combinatorial structures.
\newblock In \emph{Proceedings of the Thirtieth International Joint Conference
  on Artificial Intelligence {(IJCAI)}}, pages 4940--4945, 2021.

\bibitem[Dua and Graff(2019)]{dua2019uci}
D.~Dua and C.~Graff.
\newblock Uci machine learning repository, 2017.
\newblock \emph{URL: http://archive.ics.uci.edu/ml}, 7\penalty0 (1), 2019.

\bibitem[Eissman et~al.(2018)Eissman, Levy, Shu, Bartzsch, and
  Ermon]{UAI_ermon}
S.~Eissman, D.~Levy, R.~Shu, S.~Bartzsch, and S.~Ermon.
\newblock {B}ayesian optimization and attribute adjustment.
\newblock In \emph{Proceedings of the Thirty Fourth Conference on Uncertainty
  in Artificial Intelligence}, 2018.

\bibitem[Eriksson and Jankowiak(2021)]{eriksson2021high}
D.~Eriksson and M.~Jankowiak.
\newblock High-dimensional {B}ayesian optimization with sparse axis-aligned
  subspaces.
\newblock In \emph{Uncertainty in Artificial Intelligence}, pages 493--503.
  PMLR, 2021.

\bibitem[Eriksson et~al.(2019)Eriksson, Pearce, Gardner, Turner, and
  Poloczek]{eriksson2019scalable}
D.~Eriksson, M.~Pearce, J.~Gardner, R.~D. Turner, and M.~Poloczek.
\newblock Scalable global optimization via local {B}ayesian optimization.
\newblock \emph{Advances in neural information processing systems}, 32, 2019.

\bibitem[Frazier(2018)]{frazier2018tutorial}
P.~I. Frazier.
\newblock A tutorial on {B}ayesian optimization.
\newblock \emph{ArXiv preprint}, abs/1807.02811, 2018.

\bibitem[Gardner et~al.(2017)Gardner, Guo, Weinberger, Garnett, and
  Grosse]{gardner2017discovering}
J.~Gardner, C.~Guo, K.~Weinberger, R.~Garnett, and R.~Grosse.
\newblock Discovering and exploiting additive structure for {B}ayesian
  optimization.
\newblock In \emph{Artificial Intelligence and Statistics}, pages 1311--1319.
  PMLR, 2017.

\bibitem[Gardner et~al.(2018)Gardner, Pleiss, Weinberger, Bindel, and
  Wilson]{gardner2018gpytorch}
J.~R. Gardner, G.~Pleiss, K.~Q. Weinberger, D.~Bindel, and A.~G. Wilson.
\newblock Gpytorch: Blackbox matrix-matrix {G}aussian process inference with
  {GPU} acceleration.
\newblock In \emph{Advances in Neural Information Processing Systems 31: Annual
  Conference on Neural Information Processing Systems 2018, NeurIPS 2018,
  December 3-8, 2018, Montr{\'{e}}al, Canada}, pages 7587--7597, 2018.

\bibitem[Garnett et~al.(2013)Garnett, Osborne, and Hennig]{garnett2013active}
R.~Garnett, M.~A. Osborne, and P.~Hennig.
\newblock Active learning of linear embeddings for {G}aussian processes.
\newblock \emph{arXiv preprint arXiv:1310.6740}, 2013.

\bibitem[Garrido-Merch{\'a}n and
  Hern{\'a}ndez-Lobato(2020)]{GarridoMerchn2020DealingWC}
E.~C. Garrido-Merch{\'a}n and D.~Hern{\'a}ndez-Lobato.
\newblock Dealing with categorical and integer-valued variables in {B}ayesian
  optimization with {G}aussian processes.
\newblock \emph{Neurocomputing}, 380:\penalty0 20--35, 2020.

\bibitem[G{\'o}mez-Bombarelli et~al.(2018)G{\'o}mez-Bombarelli, Wei, Duvenaud,
  Hern{\'a}ndez-Lobato, S{\'a}nchez-Lengeling, Sheberla, Aguilera-Iparraguirre,
  Hirzel, Adams, and Aspuru-Guzik]{gomez2018}
R.~G{\'o}mez-Bombarelli, J.~N. Wei, D.~Duvenaud, J.~M. Hern{\'a}ndez-Lobato,
  B.~S{\'a}nchez-Lengeling, D.~Sheberla, J.~Aguilera-Iparraguirre, T.~D.
  Hirzel, R.~P. Adams, and A.~Aspuru-Guzik.
\newblock Automatic chemical design using a data-driven continuous
  representation of molecules.
\newblock \emph{ACS Central Science}, 4\penalty0 (2):\penalty0 268--276, 2018.

\bibitem[Guyon and Elisseeff(2003)]{guyon2003introduction}
I.~Guyon and A.~Elisseeff.
\newblock An introduction to variable and feature selection.
\newblock \emph{Journal of machine learning research}, 3\penalty0
  (Mar):\penalty0 1157--1182, 2003.

\bibitem[Hadamard(1893)]{hadamard1893resolution}
J.~Hadamard.
\newblock Resolution d'une question relative aux determinants.
\newblock \emph{Bull. des sciences math.}, 2:\penalty0 240--246, 1893.

\bibitem[Hedayat and Wallis(1978)]{hedayat1978hadamard}
A.~Hedayat and W.~D. Wallis.
\newblock Hadamard matrices and their applications.
\newblock \emph{The Annals of Statistics}, pages 1184--1238, 1978.

\bibitem[Hellsten et~al.(2022)Hellsten, Souza, Lenfers, Lacouture, Hsu, Ejjeh,
  Kjolstad, Steuwer, Olukotun, and Nardi]{hellsten2022baco}
E.~Hellsten, A.~Souza, J.~Lenfers, R.~Lacouture, O.~Hsu, A.~Ejjeh, F.~Kjolstad,
  M.~Steuwer, K.~Olukotun, and L.~Nardi.
\newblock Baco: A fast and portable {B}ayesian compiler optimization framework.
\newblock \emph{arXiv preprint arXiv:2212.11142}, 2022.

\bibitem[Horadam(2012)]{horadam2012hadamard}
K.~J. Horadam.
\newblock \emph{Hadamard matrices and their applications}.
\newblock Princeton university press, 2012.

\bibitem[Hutter et~al.(2011)Hutter, Hoos, and Leyton-Brown]{smac}
F.~Hutter, H.~H. Hoos, and K.~Leyton-Brown.
\newblock Sequential model-based optimization for general algorithm
  configuration.
\newblock In \emph{Proceedings of the 5th International Conference on Learning
  and Intelligent Optimization}, page 507–523. Springer-Verlag, 2011.
\newblock ISBN 9783642255656.

\bibitem[Kajino(2019)]{kajino_molecular_2019}
H.~Kajino.
\newblock Molecular hypergraph grammar with its application to molecular
  optimization.
\newblock In \emph{International Conference on Machine Learning}, pages
  3183--3191. PMLR, 2019.

\bibitem[Kandasamy et~al.(2015)Kandasamy, Schneider, and
  P{\'o}czos]{kandasamy2015high}
K.~Kandasamy, J.~Schneider, and B.~P{\'o}czos.
\newblock High dimensional {B}ayesian optimisation and bandits via additive
  models.
\newblock In \emph{International conference on machine learning}, pages
  295--304. PMLR, 2015.

\bibitem[Kim et~al.(2022)Kim, Choi, and Cho]{kim2022combinatorial}
J.~Kim, S.~Choi, and M.~Cho.
\newblock Combinatorial {B}ayesian optimization with random mapping functions
  to convex polytopes.
\newblock In \emph{Uncertainty in Artificial Intelligence}, pages 1001--1011.
  PMLR, 2022.

\bibitem[Kirschner et~al.(2019)Kirschner, Mutny, Hiller, Ischebeck, and
  Krause]{kirschner2019adaptive}
J.~Kirschner, M.~Mutny, N.~Hiller, R.~Ischebeck, and A.~Krause.
\newblock Adaptive and safe {B}ayesian optimization in high dimensions via
  one-dimensional subspaces.
\newblock In \emph{Proceedings of the 36th International Conference on Machine
  Learning}, volume~97 of \emph{Proceedings of Machine Learning Research},
  pages 3429--3438. {PMLR}, 2019.

\bibitem[Larsen and Nelson(2017)]{kasper2017optimality}
K.~G. Larsen and J.~Nelson.
\newblock Optimality of the johnson-lindenstrauss lemma.
\newblock In \emph{2017 IEEE 58th Annual Symposium on Foundations of Computer
  Science (FOCS)}, pages 633--638, 2017.
\newblock \doi{10.1109/FOCS.2017.64}.

\bibitem[Letham et~al.(2020)Letham, Calandra, Rai, and Bakshy]{Letham2019Re}
B.~Letham, R.~Calandra, A.~Rai, and E.~Bakshy.
\newblock Re-examining linear embeddings for high-dimensional {B}ayesian
  optimization.
\newblock In \emph{Advances in Neural Information Processing Systems 33: Annual
  Conference on Neural Information Processing Systems 2020, NeurIPS 2020,
  December 6-12, 2020, virtual}, 2020.

\bibitem[Mallat(1989)]{mallat1989theory}
S.~G. Mallat.
\newblock A theory for multiresolution signal decomposition: the wavelet
  representation.
\newblock \emph{IEEE transactions on pattern analysis and machine
  intelligence}, 11\penalty0 (7):\penalty0 674--693, 1989.

\bibitem[Maus et~al.(2022)Maus, Jones, Moore, Kusner, Bradshaw, and
  Gardner]{LOL-BO}
N.~Maus, H.~T. Jones, J.~S. Moore, M.~J. Kusner, J.~Bradshaw, and J.~R.
  Gardner.
\newblock Local latent space {B}ayesian optimization over structured inputs.
\newblock \emph{CoRR}, abs/2201.11872, 2022.

\bibitem[Nayebi et~al.(2019)Nayebi, Munteanu, and
  Poloczek]{nayebi2019framework}
A.~Nayebi, A.~Munteanu, and M.~Poloczek.
\newblock A framework for {B}ayesian optimization in embedded subspaces.
\newblock In \emph{International Conference on Machine Learning}, pages
  4752--4761. PMLR, 2019.

\bibitem[Notin et~al.(2021)Notin, Hern{\'a}ndez-Lobato, and
  Gal]{notin2021improving}
P.~Notin, J.~M. Hern{\'a}ndez-Lobato, and Y.~Gal.
\newblock Improving black-box optimization in vae latent space using decoder
  uncertainty.
\newblock \emph{arXiv preprint arXiv:2107.00096}, 2021.

\bibitem[Oh et~al.(2018)Oh, Gavves, and Welling]{BOCK}
C.~Oh, E.~Gavves, and M.~Welling.
\newblock {BOCK} : {B}ayesian optimization with cylindrical kernels.
\newblock In J.~G. Dy and A.~Krause, editors, \emph{Proceedings of the 35th
  International Conference on Machine Learning {(ICML)}}, volume~80 of
  \emph{Proceedings of Machine Learning Research}, pages 3865--3874. {PMLR},
  2018.

\bibitem[Oh et~al.(2019)Oh, Tomczak, Gavves, and Welling]{oh2019combinatorial}
C.~Oh, J.~M. Tomczak, E.~Gavves, and M.~Welling.
\newblock Combinatorial {B}ayesian optimization using the graph cartesian
  product.
\newblock In \emph{Advances in Neural Information Processing Systems 32: Annual
  Conference on Neural Information Processing Systems 2019, NeurIPS 2019,
  December 8-14, 2019, Vancouver, BC, Canada}, pages 2910--2920, 2019.

\bibitem[Oh et~al.(2021)Oh, Gavves, and Welling]{oh2021mixed}
C.~Oh, E.~Gavves, and M.~Welling.
\newblock Mixed variable {B}ayesian optimization with frequency modulated
  kernels.
\newblock \emph{ArXiv preprint}, abs/ 2102, 2021.

\bibitem[Packebusch and Mertens(2015)]{LABS_main}
T.~Packebusch and S.~Mertens.
\newblock Low autocorrelation binary sequences.
\newblock \emph{Journal of Physics A: Mathematical and Theoretical, 49 (2016)
  165001}, 2015.
\newblock \doi{10.1088/1751-8113/49/16/165001}.

\bibitem[Papenmeier et~al.(2022)Papenmeier, Nardi, and Poloczek]{NeurIPS2022}
L.~Papenmeier, L.~Nardi, and M.~Poloczek.
\newblock Increasing the scope as you learn: Adaptive {B}ayesian optimization
  in nested subspaces.
\newblock In \emph{Advances in Neural Information Processing Systems
  (NeurIPS)}, 2022.

\bibitem[Rasmussen(2004)]{Rasmussen2004}
C.~E. Rasmussen.
\newblock {G}aussian processes in machine learning.
\newblock In \emph{Advanced Lectures on Machine Learning: ML Summer Schools
  2003, Canberra, Australia, February 2 - 14, 2003, T{\"u}bingen, Germany,
  August 4 - 16, 2003, Revised Lectures}, 2004.

\bibitem[Ru et~al.(2020)Ru, Alvi, Nguyen, Osborne, and Roberts]{ru2020bayesian}
B.~X. Ru, A.~S. Alvi, V.~Nguyen, M.~A. Osborne, and S.~J. Roberts.
\newblock {B}ayesian optimisation over multiple continuous and categorical
  inputs.
\newblock In \emph{Proc. of ICML}, volume 119 of \emph{Proceedings of Machine
  Learning Research}, pages 8276--8285. {PMLR}, 2020.

\bibitem[Rudin(1974)]{rudin1974real}
W.~Rudin.
\newblock Real and complex analysis, mcgraw-hill.
\newblock \emph{Inc.,}, 1974.

\bibitem[Shapiro et~al.(1968)Shapiro, Pettengill, Ash, Stone, Smith, Ingalls,
  and Brockelman]{LABS_communication}
I.~I. Shapiro, G.~H. Pettengill, M.~E. Ash, M.~L. Stone, W.~B. Smith, R.~P.
  Ingalls, and R.~A. Brockelman.
\newblock Fourth test of general relativity: preliminary results.
\newblock \emph{Physical Review Letters}, 20\penalty0 (22):\penalty0 1265,
  1968.

\bibitem[Srinivas et~al.(2010)Srinivas, Krause, Kakade, and Seeger]{ucb}
N.~Srinivas, A.~Krause, S.~M. Kakade, and M.~W. Seeger.
\newblock {G}aussian process optimization in the bandit setting: No regret and
  experimental design.
\newblock In \emph{Proc. of ICML}, pages 1015--1022. Omnipress, 2010.

\bibitem[Swanson and Tewfik(1996)]{swanson1996binary}
M.~D. Swanson and A.~H. Tewfik.
\newblock A binary wavelet decomposition of binary images.
\newblock \emph{IEEE Transactions on Image Processing}, 5\penalty0
  (12):\penalty0 1637--1650, 1996.

\bibitem[Tripp et~al.(2020)Tripp, Daxberger, and
  Hern{\'a}ndez-Lobato]{tripp2020sample}
A.~Tripp, E.~Daxberger, and J.~M. Hern{\'a}ndez-Lobato.
\newblock Sample-efficient optimization in the latent space of deep generative
  models via weighted retraining.
\newblock \emph{Advances in Neural Information Processing Systems}, 33, 2020.

\bibitem[Tropp(2004)]{tropp2004greed}
J.~A. Tropp.
\newblock Greed is good: Algorithmic results for sparse approximation.
\newblock \emph{IEEE Transactions on Information theory}, 50\penalty0
  (10):\penalty0 2231--2242, 2004.

\bibitem[Wan et~al.(2021)Wan, Nguyen, Ha, Ru, Lu, and Osborne]{wan2021think}
X.~Wan, V.~Nguyen, H.~Ha, B.~X. Ru, C.~Lu, and M.~A. Osborne.
\newblock Think global and act local: {B}ayesian optimisation over
  high-dimensional categorical and mixed search spaces.
\newblock In \emph{Proc. of ICML}, volume 139 of \emph{Proceedings of Machine
  Learning Research}, pages 10663--10674. {PMLR}, 2021.

\bibitem[Wang et~al.(2016)Wang, Hutter, Zoghi, Matheson, and
  De~Feitas]{wang2016bayesian}
Z.~Wang, F.~Hutter, M.~Zoghi, D.~Matheson, and N.~De~Feitas.
\newblock {B}ayesian optimization in a billion dimensions via random
  embeddings.
\newblock \emph{Journal of Artificial Intelligence Research}, 55:\penalty0
  361--387, 2016.

\bibitem[Zhang et~al.(2021)Zhang, Chang, Li, Wu, Tan, Li, and
  Cui]{zhang2021facilitating}
X.~Zhang, Z.~Chang, Y.~Li, H.~Wu, J.~Tan, F.~Li, and B.~Cui.
\newblock Facilitating database tuning with hyper-parameter optimization: a
  comprehensive experimental evaluation.
\newblock \emph{arXiv preprint arXiv:2110.12654}, 2021.

\end{thebibliography}


\begin{thebibliography}{53}
\providecommand{\natexlab}[1]{#1}
\providecommand{\url}[1]{\texttt{#1}}
\expandafter\ifx\csname urlstyle\endcsname\relax
  \providecommand{\doi}[1]{doi: #1}\else
  \providecommand{\doi}{doi: \begingroup \urlstyle{rm}\Url}\fi

\bibitem[Baptista and Poloczek(2018)]{baptista2018bayesian}
R.~Baptista and M.~Poloczek.
\newblock Bayesian optimization of combinatorial structures.
\newblock In \emph{International Conference on Machine Learning}, pages
  462--471. PMLR, 2018.

\bibitem[Bergstra et~al.(2011)Bergstra, Bardenet, Bengio, and
  K{\'e}gl]{bergstra2011algorithms}
J.~Bergstra, R.~Bardenet, Y.~Bengio, and B.~K{\'e}gl.
\newblock Algorithms for hyper-parameter optimization.
\newblock \emph{Advances in neural information processing systems}, 24, 2011.

\bibitem[Bliek et~al.(2021)Bliek, Guijt, Verwer, and De~Weerdt]{bliek2021black}
L.~Bliek, A.~Guijt, S.~Verwer, and M.~De~Weerdt.
\newblock Black-box mixed-variable optimisation using a surrogate model that
  satisfies integer constraints.
\newblock In \emph{Proceedings of the Genetic and Evolutionary Computation
  Conference Companion}, pages 1851--1859, 2021.

\bibitem[Box and Wilson(1992)]{box1992experimental}
G.~E. Box and K.~B. Wilson.
\newblock On the experimental attainment of optimum conditions.
\newblock In \emph{Breakthroughs in statistics: methodology and distribution},
  pages 270--310. Springer, 1992.

\bibitem[Boyd and Vandenberghe(2004)]{boyd2004convex}
S.~Boyd and L.~Vandenberghe.
\newblock \emph{Convex optimization}.
\newblock Cambridge university press, 2004.

\bibitem[Breiman(2001)]{breiman2001random}
L.~Breiman.
\newblock Random forests.
\newblock \emph{Machine learning}, 45:\penalty0 5--32, 2001.

\bibitem[Brochu et~al.(2010)Brochu, Cora, and De~Freitas]{brochu2010tutorial}
E.~Brochu, V.~M. Cora, and N.~De~Freitas.
\newblock A tutorial on bayesian optimization of expensive cost functions, with
  application to active user modeling and hierarchical reinforcement learning.
\newblock \emph{arXiv preprint arXiv:1012.2599}, 2010.

\bibitem[Dai et~al.(2024)Dai, Nguyen, Tay, Urano, Leong, Low, and
  Jaillet]{dai2024batch}
Z.~Dai, Q.~P. Nguyen, S.~Tay, D.~Urano, R.~Leong, B.~K.~H. Low, and P.~Jaillet.
\newblock Batch bayesian optimization for replicable experimental design.
\newblock \emph{Advances in Neural Information Processing Systems}, 36, 2024.

\bibitem[Daxberger et~al.(2019)Daxberger, Makarova, Turchetta, and
  Krause]{daxberger2019mixed}
E.~Daxberger, A.~Makarova, M.~Turchetta, and A.~Krause.
\newblock Mixed-variable bayesian optimization.
\newblock \emph{arXiv preprint arXiv:1907.01329}, 2019.

\bibitem[Deshwal and Doppa(2021)]{deshwal2021combining}
A.~Deshwal and J.~Doppa.
\newblock Combining latent space and structured kernels for bayesian
  optimization over combinatorial spaces.
\newblock \emph{Advances in Neural Information Processing Systems},
  34:\penalty0 8185--8200, 2021.

\bibitem[Deshwal et~al.()Deshwal, Ament, Balandat, Bakshy, Doppa, and
  Eriksson]{deshwal2303bayesian}
A.~Deshwal, S.~Ament, M.~Balandat, E.~Bakshy, J.~R. Doppa, and D.~Eriksson.
\newblock Bayesian optimization over high-dimensional combinatorial spaces via
  dictionary-based embeddings. corr, abs/2303.01774, 2023. doi: 10.48550.
\newblock \emph{arXiv preprint arXiv.2303.01774}, 10.

\bibitem[Deshwal et~al.(2023)Deshwal, Ament, Balandat, Bakshy, Doppa, and
  Eriksson]{deshwal2023bayesian}
A.~Deshwal, S.~Ament, M.~Balandat, E.~Bakshy, J.~R. Doppa, and D.~Eriksson.
\newblock Bayesian optimization over high-dimensional combinatorial spaces via
  dictionary-based embeddings.
\newblock In \emph{International Conference on Artificial Intelligence and
  Statistics}, pages 7021--7039. PMLR, 2023.

\bibitem[Deza et~al.(1997)Deza, Laurent, and Weismantel]{deza1997geometry}
M.~Deza, M.~Laurent, and R.~Weismantel.
\newblock \emph{Geometry of cuts and metrics}, volume~2.
\newblock Springer, 1997.

\bibitem[Dreczkowski et~al.(2024)Dreczkowski, Grosnit, and
  Bou~Ammar]{dreczkowski2024framework}
K.~Dreczkowski, A.~Grosnit, and H.~Bou~Ammar.
\newblock Framework and benchmarks for combinatorial and mixed-variable
  bayesian optimization.
\newblock \emph{Advances in Neural Information Processing Systems}, 36, 2024.

\bibitem[Eriksson and Jankowiak(2021)]{eriksson2021high}
D.~Eriksson and M.~Jankowiak.
\newblock High-dimensional bayesian optimization with sparse axis-aligned
  subspaces.
\newblock In \emph{Uncertainty in Artificial Intelligence}, pages 493--503.
  PMLR, 2021.

\bibitem[Frazier(2018)]{frazier2018tutorial}
P.~I. Frazier.
\newblock A tutorial on bayesian optimization.
\newblock \emph{arXiv preprint arXiv:1807.02811}, 2018.

\bibitem[Gardner et~al.(2018)Gardner, Pleiss, Weinberger, Bindel, and
  Wilson]{gardner2018gpytorch}
J.~R. Gardner, G.~Pleiss, K.~Q. Weinberger, D.~Bindel, and A.~G. Wilson.
\newblock Gpytorch: Blackbox matrix-matrix {G}aussian process inference with
  {GPU} acceleration.
\newblock In \emph{Advances in Neural Information Processing Systems 31: Annual
  Conference on Neural Information Processing Systems 2018, NeurIPS 2018,
  December 3-8, 2018, Montr{\'{e}}al, Canada}, pages 7587--7597, 2018.

\bibitem[Garrido-Merch{\'a}n and
  Hern{\'a}ndez-Lobato(2020)]{garrido2020dealing}
E.~C. Garrido-Merch{\'a}n and D.~Hern{\'a}ndez-Lobato.
\newblock Dealing with categorical and integer-valued variables in bayesian
  optimization with gaussian processes.
\newblock \emph{Neurocomputing}, 380:\penalty0 20--35, 2020.

\bibitem[Golovin et~al.(2017)Golovin, Solnik, Moitra, Kochanski, Karro, and
  Sculley]{golovin2017google}
D.~Golovin, B.~Solnik, S.~Moitra, G.~Kochanski, J.~Karro, and D.~Sculley.
\newblock Google vizier: A service for black-box optimization.
\newblock In \emph{Proceedings of the 23rd ACM SIGKDD international conference
  on knowledge discovery and data mining}, pages 1487--1495, 2017.

\bibitem[Gonz{\'a}lez et~al.(2016)Gonz{\'a}lez, Dai, Hennig, and
  Lawrence]{gonzalez2016batch}
J.~Gonz{\'a}lez, Z.~Dai, P.~Hennig, and N.~Lawrence.
\newblock Batch bayesian optimization via local penalization.
\newblock In \emph{Artificial intelligence and statistics}, pages 648--657.
  PMLR, 2016.

\bibitem[Halstrup(2016)]{halstrup2016black}
M.~Halstrup.
\newblock \emph{Black-box optimization of mixed discrete-continuous
  optimization problems}.
\newblock PhD thesis, Dissertation, Dortmund, Technische Universit{\"a}t, 2016,
  2016.

\bibitem[Hern{\'a}ndez-Lobato et~al.(2017)Hern{\'a}ndez-Lobato, Requeima,
  Pyzer-Knapp, and Aspuru-Guzik]{hernandez2017parallel}
J.~M. Hern{\'a}ndez-Lobato, J.~Requeima, E.~O. Pyzer-Knapp, and
  A.~Aspuru-Guzik.
\newblock Parallel and distributed thompson sampling for large-scale
  accelerated exploration of chemical space.
\newblock In \emph{International conference on machine learning}, pages
  1470--1479. PMLR, 2017.

\bibitem[Hiriart-Urruty and Lemar{\'e}chal(2004)]{hiriart2004fundamentals}
J.-B. Hiriart-Urruty and C.~Lemar{\'e}chal.
\newblock \emph{Fundamentals of convex analysis}.
\newblock Springer Science \& Business Media, 2004.

\bibitem[Hoffman et~al.(2014)Hoffman, Gelman, et~al.]{hoffman2014no}
M.~D. Hoffman, A.~Gelman, et~al.
\newblock The no-u-turn sampler: adaptively setting path lengths in hamiltonian
  monte carlo.
\newblock \emph{J. Mach. Learn. Res.}, 15\penalty0 (1):\penalty0 1593--1623,
  2014.

\bibitem[Hu et~al.(2008)Hu, Yao, and Hua]{hu2008optimization}
W.~Hu, L.~G. Yao, and Z.~Z. Hua.
\newblock Optimization of sheet metal forming processes by adaptive response
  surface based on intelligent sampling method.
\newblock \emph{Journal of materials processing technology}, 197\penalty0
  (1-3):\penalty0 77--88, 2008.

\bibitem[Hutter et~al.(2011)Hutter, Hoos, and
  Leyton-Brown]{hutter2011sequential}
F.~Hutter, H.~H. Hoos, and K.~Leyton-Brown.
\newblock Sequential model-based optimization for general algorithm
  configuration.
\newblock In \emph{Learning and Intelligent Optimization: 5th International
  Conference, LION 5, Rome, Italy, January 17-21, 2011. Selected Papers 5},
  pages 507--523. Springer, 2011.

\bibitem[Jansson et~al.(2003)Jansson, Nilsson, and Redhe]{jansson2003using}
T.~Jansson, L.~Nilsson, and M.~Redhe.
\newblock Using surrogate models and response surfaces in structural
  optimization--with application to crashworthiness design and sheet metal
  forming.
\newblock \emph{Structural and Multidisciplinary Optimization}, 25:\penalty0
  129--140, 2003.

\bibitem[Kandasamy et~al.(2018)Kandasamy, Neiswanger, Schneider, Poczos, and
  Xing]{kandasamy2018neural}
K.~Kandasamy, W.~Neiswanger, J.~Schneider, B.~Poczos, and E.~P. Xing.
\newblock Neural architecture search with bayesian optimisation and optimal
  transport.
\newblock \emph{Advances in neural information processing systems}, 31, 2018.

\bibitem[Korovina et~al.(2020)Korovina, Xu, Kandasamy, Neiswanger, Poczos,
  Schneider, and Xing]{korovina2020chembo}
K.~Korovina, S.~Xu, K.~Kandasamy, W.~Neiswanger, B.~Poczos, J.~Schneider, and
  E.~Xing.
\newblock Chembo: Bayesian optimization of small organic molecules with
  synthesizable recommendations.
\newblock In \emph{International Conference on Artificial Intelligence and
  Statistics}, pages 3393--3403. PMLR, 2020.

\bibitem[MacKay and Neal(1994)]{mackay1994automatic}
D.~J. MacKay and R.~M. Neal.
\newblock Automatic relevance determination for neural networks.
\newblock In \emph{Technical Report in preparation.} Cambridge University,
  1994.

\bibitem[Naceur et~al.(2006)Naceur, Guo, and Ben-Elechi]{naceur2006response}
H.~Naceur, Y.~Guo, and S.~Ben-Elechi.
\newblock Response surface methodology for design of sheet forming parameters
  to control springback effects.
\newblock \emph{Computers \& structures}, 84\penalty0 (26-27):\penalty0
  1651--1663, 2006.

\bibitem[Nguyen and Osborne(2020)]{nguyen2020knowing}
V.~Nguyen and M.~A. Osborne.
\newblock Knowing the what but not the where in bayesian optimization.
\newblock In \emph{International Conference on Machine Learning}, pages
  7317--7326. PMLR, 2020.

\bibitem[Nguyen et~al.(2021)Nguyen, Le, Yamada, and Osborne]{nguyen2021optimal}
V.~Nguyen, T.~Le, M.~Yamada, and M.~A. Osborne.
\newblock Optimal transport kernels for sequential and parallel neural
  architecture search.
\newblock In \emph{International Conference on Machine Learning}, pages
  8084--8095. PMLR, 2021.

\bibitem[Oune and Bostanabad(2021)]{oune2021latent}
N.~Oune and R.~Bostanabad.
\newblock Latent map gaussian processes for mixed variable metamodeling.
\newblock \emph{Computer Methods in Applied Mechanics and Engineering},
  387:\penalty0 114128, 2021.

\bibitem[Papenmeier et~al.(2023)Papenmeier, Nardi, and
  Poloczek]{papenmeier2023bounce}
L.~Papenmeier, L.~Nardi, and M.~Poloczek.
\newblock Bounce: reliable high-dimensional bayesian optimization for
  combinatorial and mixed spaces.
\newblock \emph{Advances in Neural Information Processing Systems},
  36:\penalty0 1764--1793, 2023.

\bibitem[Qian et~al.(2008)Qian, Wu, and Wu]{qian2008gaussian}
P.~Z.~G. Qian, H.~Wu, and C.~J. Wu.
\newblock Gaussian process models for computer experiments with qualitative and
  quantitative factors.
\newblock \emph{Technometrics}, 50\penalty0 (3):\penalty0 383--396, 2008.

\bibitem[Rasmussen(2003)]{rasmussen2003gaussian}
C.~E. Rasmussen.
\newblock Gaussian processes in machine learning.
\newblock In \emph{Summer school on machine learning}, pages 63--71. Springer,
  2003.

\bibitem[Ru et~al.(2020{\natexlab{a}})Ru, Alvi, Nguyen, Osborne, and
  Roberts]{ru2020bayesian}
B.~Ru, A.~Alvi, V.~Nguyen, M.~A. Osborne, and S.~Roberts.
\newblock Bayesian optimisation over multiple continuous and categorical
  inputs.
\newblock In \emph{International Conference on Machine Learning}, pages
  8276--8285. PMLR, 2020{\natexlab{a}}.

\bibitem[Ru et~al.(2020{\natexlab{b}})Ru, Wan, Dong, and
  Osborne]{ru2020interpretable}
B.~Ru, X.~Wan, X.~Dong, and M.~Osborne.
\newblock Interpretable neural architecture search via bayesian optimisation
  with weisfeiler-lehman kernels.
\newblock \emph{arXiv preprint arXiv:2006.07556}, 2020{\natexlab{b}}.

\bibitem[Saves et~al.(2023)Saves, Diouane, Bartoli, Lefebvre, and
  Morlier]{saves2023mixed}
P.~Saves, Y.~Diouane, N.~Bartoli, T.~Lefebvre, and J.~Morlier.
\newblock A mixed-categorical correlation kernel for gaussian process.
\newblock \emph{Neurocomputing}, 550:\penalty0 126472, 2023.

\bibitem[Scannell et~al.(2023)Scannell, Ek, and Richards]{scannell2023mode}
A.~Scannell, C.~H. Ek, and A.~Richards.
\newblock Mode-constrained model-based reinforcement learning via gaussian
  processes.
\newblock In \emph{International Conference on Artificial Intelligence and
  Statistics}, pages 3299--3314. PMLR, 2023.

\bibitem[Schoenberg(1935)]{schoenberg1935remarks}
I.~J. Schoenberg.
\newblock Remarks to maurice frechet's article``sur la definition axiomatique
  d'une classe d'espace distances vectoriellement applicable sur l'espace de
  hilbert.
\newblock \emph{Annals of Mathematics}, 36\penalty0 (3):\penalty0 724--732,
  1935.

\bibitem[Shahriari et~al.(2015)Shahriari, Swersky, Wang, Adams, and
  De~Freitas]{shahriari2015taking}
B.~Shahriari, K.~Swersky, Z.~Wang, R.~P. Adams, and N.~De~Freitas.
\newblock Taking the human out of the loop: A review of bayesian optimization.
\newblock \emph{Proceedings of the IEEE}, 104\penalty0 (1):\penalty0 148--175,
  2015.

\bibitem[Snoek et~al.(2012)Snoek, Larochelle, and Adams]{snoek2012practical}
J.~Snoek, H.~Larochelle, and R.~P. Adams.
\newblock Practical bayesian optimization of machine learning algorithms.
\newblock \emph{Advances in neural information processing systems}, 25, 2012.

\bibitem[Stein(2012)]{stein2012interpolation}
M.~L. Stein.
\newblock \emph{Interpolation of spatial data: some theory for kriging}.
\newblock Springer Science \& Business Media, 2012.

\bibitem[Stephenson et~al.(2022)Stephenson, Ghosh, Nguyen, Yurochkin,
  Deshpande, and Broderick]{stephenson2022measuring}
W.~T. Stephenson, S.~Ghosh, T.~D. Nguyen, M.~Yurochkin, S.~Deshpande, and
  T.~Broderick.
\newblock Measuring the robustness of gaussian processes to kernel choice.
\newblock In \emph{International Conference on Artificial Intelligence and
  Statistics}, pages 3308--3331. PMLR, 2022.

\bibitem[Tuo and Wang(2022)]{tuo2022uncertainty}
R.~Tuo and W.~Wang.
\newblock Uncertainty quantification for bayesian optimization.
\newblock In \emph{International Conference on Artificial Intelligence and
  Statistics}, pages 2862--2884. PMLR, 2022.

\bibitem[Wan et~al.(2021)Wan, Nguyen, Ha, Ru, Lu, and Osborne]{wan2021think}
X.~Wan, V.~Nguyen, H.~Ha, B.~Ru, C.~Lu, and M.~A. Osborne.
\newblock Think global and act local: Bayesian optimisation over
  high-dimensional categorical and mixed search spaces.
\newblock In \emph{International Conference on Machine Learning}, pages
  10663--10674. PMLR, 2021.

\bibitem[Wei and Yuying(2008)]{wei2008multi}
L.~Wei and Y.~Yuying.
\newblock Multi-objective optimization of sheet metal forming process using
  pareto-based genetic algorithm.
\newblock \emph{Journal of materials processing technology}, 208\penalty0
  (1-3):\penalty0 499--506, 2008.

\bibitem[Zaefferer(2018)]{zaefferer2018surrogate}
M.~Zaefferer.
\newblock Surrogate models for discrete optimization problems.
\newblock 2018.

\bibitem[Zhang and Notz(2015)]{zhang2015computer}
Y.~Zhang and W.~I. Notz.
\newblock Computer experiments with qualitative and quantitative variables: A
  review and reexamination.
\newblock \emph{Quality Engineering}, 27\penalty0 (1):\penalty0 2--13, 2015.

\bibitem[Zhang et~al.(2020)Zhang, Tao, Chen, and Apley]{zhang2020latent}
Y.~Zhang, S.~Tao, W.~Chen, and D.~W. Apley.
\newblock A latent variable approach to gaussian process modeling with
  qualitative and quantitative factors.
\newblock \emph{Technometrics}, 62\penalty0 (3):\penalty0 291--302, 2020.

\bibitem[Zhou et~al.(2011)Zhou, Qian, and Zhou]{zhou2011simple}
Q.~Zhou, P.~Z. Qian, and S.~Zhou.
\newblock A simple approach to emulation for computer models with qualitative
  and quantitative factors.
\newblock \emph{Technometrics}, 53\penalty0 (3):\penalty0 266--273, 2011.

\end{thebibliography}

\clearpage
\onecolumn

\hsize\textwidth\linewidth\hsize\toptitlebar
{\centering{\Large\bfseries Weighted Euclidean Distance Matrices over Mixed Continuous and Categorical Inputs for Gaussian Process Models \\ Supplementary Materials \par}}
\bottomtitlebar
\appendix
\section{Theoretical analysis of EDM}
\label{sec:linear}
\subsection{Proof of proposition \ref{proposition:linearindependent}}
\label{subsec:edm}
To establish that $D$ is also a valid EDM, we utilize the characterization of EDMs through their associated Gram matrices. Recall that for any EDM $D$, there exists a Gram matrix $G$ such that
$$
G=-\frac{1}{2} H D H,
$$
where $H=I-\frac{1}{n} \mathbf{e} \mathbf{e}^{\top}$ is the centering matrix, $I$ is the identity matrix, and $\mathbf{e}$ is the column vector of ones. The matrix $G$ is symmetric positive semidefinite (PSD) if and only if $D$ is an EDM.
For each $D^{(i)}$, since it is an EDM, the corresponding Gram matrix $G^{(i)}$ is given by
$$
G^{(i)}=-\frac{1}{2} H D^{(i)} H
$$
and $G^{(i)}$ is PSD. Consider the Gram matrix associated with $D$ :
$$
G=-\frac{1}{2} H D H=-\frac{1}{2} H\left(\sum_{i=1}^m w_i D^{(i)}\right) H=\sum_{i=1}^m w_i\left(-\frac{1}{2} H D^{(i)} H\right)=\sum_{i=1}^m w_i G^{(i)} .
$$
Since each $G^{(i)}$ is PSD and $w_i \geq 0$, their weighted sum $G$ is also PSD. This implies that there exists a configuration of points $\left\{x_1, x_2, \ldots, x_n\right\}$ in a Euclidean space such that $D_{j k}=\| x_j-$ $x_k \|^2$ for all $j, k$. Therefore, $D$ is a valid Euclidean distance matrix, as it satisfies the necessary condition through its PSD Gram matrix $G$.

\subsection{Proof of proposition \ref{prop:extreme}}
It can be directly derived from definition \ref{form:extreme direction} and lemma \ref{lemma:cara}.
\label{subsec:any_edm}

\section{Theoretical analysis of WEGP}
\label{sec:converge}

In theoretical analysis, we discuss the Matern kernel. Theorem \ref{main theorem} is derived from the Matern kernel.

\begin{definition}[WEGP with Matern Kernel]
Our kernel is defined by modifying the distance measure of the Matérn kernel. Our kernel function $K\left(\mathbf{z}_p, \mathbf{z}_q \mid \sigma, \boldsymbol{\theta},\left\{\boldsymbol{w}_k\right\}\right) $ is defined as:
\begin{equation}
\label{eq:kernel}
    K\left(\mathbf{z}_p, \mathbf{z}_q \mid \sigma, \boldsymbol{\theta},\left\{\boldsymbol{w}_k\right\}\right)=\sigma_0^2 \frac{1}{\Gamma(\nu) 2^{\nu-1}}\left(\frac{d\left(\mathbf{z}_p, \mathbf{z}_q\right)}{\theta}\right)^\nu K_\nu\left(\frac{d\left(\mathbf{z}_p, \mathbf{z}_q\right)}{\theta}\right)
\end{equation}
where $\sigma_0^2$ is the process variance parameter; $d(\mathbf{z}_p, \mathbf{z}_q)$ is the distance between the inputs $\mathbf{z}_p$ and $\mathbf{z}_q$, defined as:
\begin{equation}
\label{eq:distance}
d\left(\mathbf{z}_p, \mathbf{z}_q\right)=\sqrt{d_{\mathbf{x}}\left(\mathbf{x}_p, \mathbf{x}_q\right)^2+\sum_{k=1}^c D_{k,}{}_{p q}},
\end{equation}
where $d_{\mathbf{x}}\left(\mathbf{x}_p, \mathbf{x}_q\right)$ is the Euclidean distance between continuous input, $D_{k,}{}_{p q}$ is the estimated Euclidean distance between categories $\mathbf{h}_p^{(k)}$ and $\mathbf{h}_q^{(k)}$.
\end{definition}

\begin{lemma}
\label{lemma:semi-defi}
    The kernel function $K(\mathbf{z}_p, \mathbf{z}_q)$ is positive semi-definite.
\end{lemma}

\begin{proof}
According to the definition of kernel in Eq. \ref{eq:kernel}, by Bochner's theorem~\citep{rasmussen2003gaussian}, the Matern kernel on $\mathbf{x}_p$ is positive semidefinite.
For the categorical part, each categorical variable is equipped with valid Euclidean distances $D_{k, p q}^{(i)}$ . That means we can embed the categorical variable with $c_k$ choices into Euclidean space $\mathbb{R}^{c_k(c_k-1)/2}$ such that
$$
\sum_{i=1}^{m_k} w_k^{(i)} D_{k, p q}^{(i)} \longleftrightarrow\left\|\Psi_k\left(\mathbf{h}_p^{(k)}\right)-\Psi_k\left(\mathbf{h}_q^{(k)}\right)\right\|^2
$$
up to a constant scaling factor. By this transformation, we can view the kernel be the continuous kernel with inputs $\Psi_k\left(\mathbf{h}_p^{(k)}\right)$. Such a kernel is positive semidefinite.
Since the product of positive semidefinite kernels remains positive semidefinite, multiplying the continuous part by the categorical parts preserves positive semidefiniteness.
\end{proof}
\begin{proof}[Proof of Theorem \ref{main theorem}]
Since the GP posterior mean is the optimal interpolant in $\mathcal{H}_\nu(\mathcal{Z})$, we have

$$
\left|f(\mathbf{z})-\mu_n(\mathbf{z}; \nu)\right| \leq s_n(\mathbf{z}; \nu)\|f\|_{\mathcal{H}_\nu(\mathcal{Z})}
$$
By the fill distance assumption, for every $\mathbf{z} \in \mathcal{Z}$ there exists a sampled point $z_n$ with
$$
\left\|\mathbf{z}-\mathbf{z_i}\right\| \leq h
$$
According to \citep{stein2012interpolation} the Matérn kernel $K_\nu$ is $C^k$ with $k$ the largest integer less than $2 \nu$ and, near the origin, admits a $k$ th-order Taylor approximation $P_k$ satisfying

$$
\left|K_\nu(\mathbf{z})-P_k(\mathbf{z})\right|=O\left(\|\mathbf{z}\|^{2 \nu}(-\log \|\mathbf{z}\|)^{2 \alpha}\right) \quad \text { as } z \rightarrow 0
$$

for some $\alpha \geq 0$. In particular, since $K_\nu(0)=1$, for small $r=\left\|\mathbf{z}-\mathbf{z_i}\right\|$ we have
$$
1-K_\nu\left(\mathbf{z},\mathbf{z_i}\right)=O\left(r^{2(\nu \wedge 1)}(-\log r)^{2 \alpha}\right)
$$

Taking $r \leq h$ yields

$$
1-K_\nu\left(\mathbf{z},\mathbf{z_i}\right)=O\left(h^{2(\nu \wedge 1)}\left(\log \frac{1}{h}\right)^{2 \alpha}\right)
$$
Standard Kriging theory then implies
$$
s_n^2(\mathbf{z} ; \nu) \leq 2\left(1-K_\nu\left(\mathbf{z},\mathbf{z_i}\right)\right) \leq C_1 h^{2(\nu \wedge 1)}\left(\log \frac{1}{h}\right)^{2 \alpha}
$$
Taking square roots, we obtain

$$
s_n(\mathbf{z}; \nu) \leq \sqrt{C_1} h^{\nu \wedge 1}\left(\log \frac{1}{h}\right)^\alpha
$$

Setting $\beta=\alpha$ (or an equivalent exponent as determined by the precise kernel properties) completes the proof.
\end{proof}

\section{Model performance measured by log likelihood}
\label{sec:log_likelihood}
\begin{figure*}[h]
    \centering

    \includegraphics[width=0.24\textwidth]{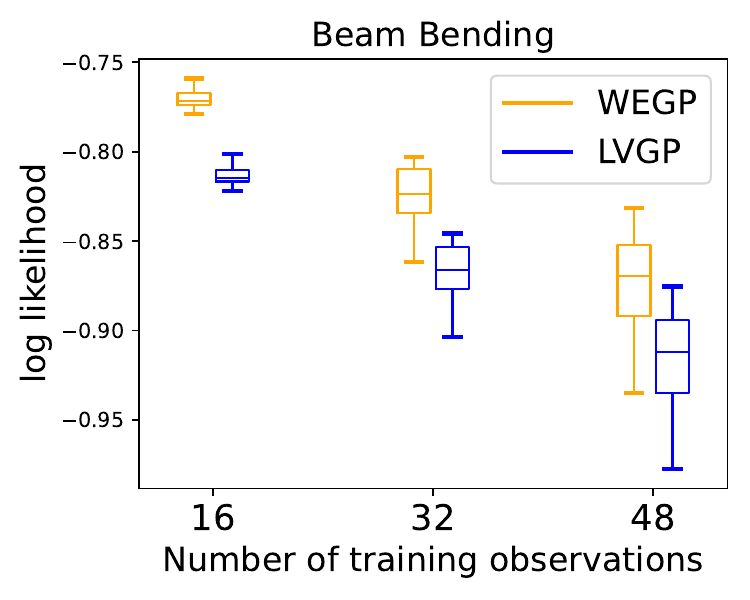}
    \includegraphics[width=0.24\textwidth]{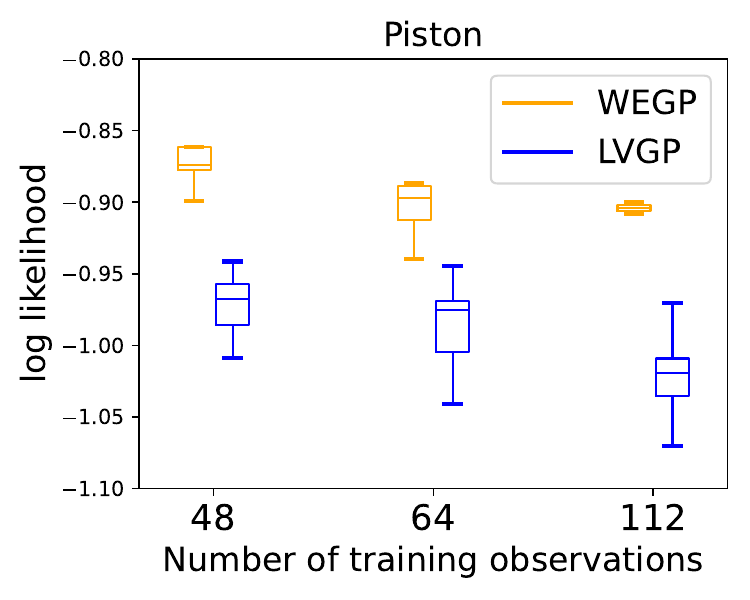}
    \includegraphics[width=0.24\textwidth]{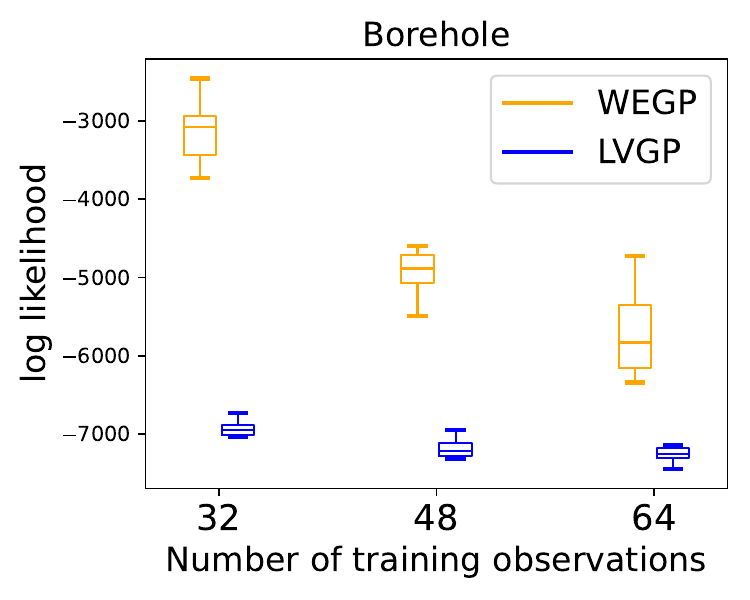}
    \includegraphics[width=0.24\textwidth]{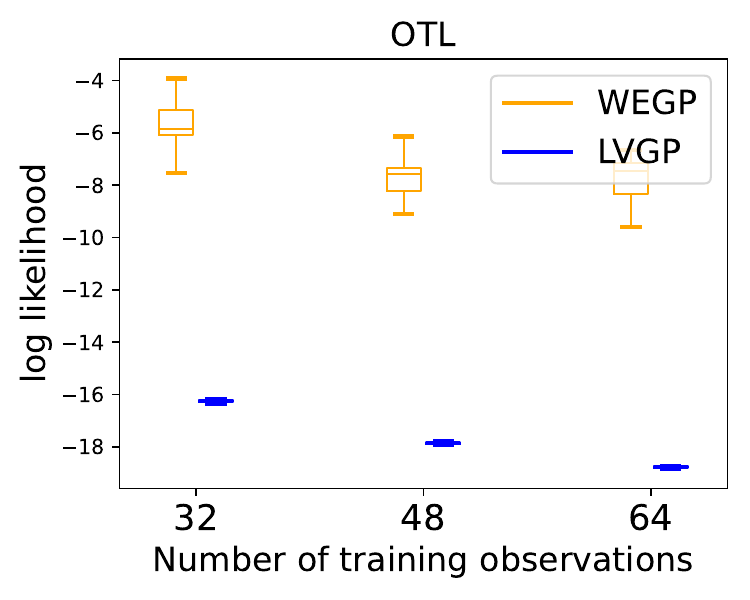}

    \caption{WEGP model accuracy comparison by log likelihood by three different sizes of training data}
    \label{fig:ll}
\end{figure*}
Figure \ref{fig:ll} compares the performance LVGP and WEGP using log likelihood. A higher log likelihood indicates a better fit to the data, suggesting lower predictive uncertainty and higher predictive accuracy.

From the four benchmark problems—Beam Bending, Piston, Borehole, and OTL—evaluated at different training set sizes, we observe that WEGP consistently achieves higher log likelihood values compared to LVGP. This indicates that WEGP provides a better fit and higher predictive performance on these test cases.

\section{Test functions}
\label{sec:test_function}

\subsection{Test functions of comparing model accuracy}
\label{subsec:model test}
\subsubsection{Analytical test functions}
Table \ref{table:test} summarizes the analytical functions used for comparing WEGP to LVGP. The functions possess a wide range of dimensionality and complexity. 

\begin{table}[h]
\renewcommand{\arraystretch}{1.5}
\centering

\begin{tabular}{ll}
\hline
{Function} & {Description} \\
\hline
1 - OTL Circuit & 
$y = \left( V_{b1} + 0.74 \right) \beta (R_{c2} + 9) + 11.35 R_f + \frac{0.74 R_f \beta (R_{c2} + 9)}{\beta (R_{c2} + 9) + R_f} R_{c1}$ \\
& $V_{b1} = \frac{12 R_{b2}}{R_{b1} + R_{b2}}$ \\
\hline
2 - Piston Simulator & 
$y = 2\pi \sqrt{\frac{M}{k + \frac{S^2 P_0 V_0 T}{T_0 V^2}}}$ \\
& $V = \frac{S}{2k} \left( A^2 + \frac{4kT}{T_0} \right), \quad A = P_0 S + 19.62 M - \frac{k V_0}{S}$ \\
\hline
3 - Borehole & 
$y = \frac{2 \pi T_u (H_u - H_l)}{\ln\left(\frac{r}{r_0}\right) \left( 1 + \frac{2 L T_u}{\ln\left(\frac{r}{r_0}\right) r_0^2 k_w} + \frac{T_u}{T_l} \right)}$ \\
\hline

\end{tabular}
\caption{Analytical test functions.}
\label{table:test}
\end{table}\textbf{}

The input variables for four engineering function are summarized in the Table \ref{table:config}. These tables outline the range of values that each variable can take, categorized into both quantitative and categorical inputs.

\begin{table}[h]
\renewcommand{\arraystretch}{1.5}
\centering
\begin{tabular}{c|c|c}
\hline ID & Variables  & Min, Max \\
\hline 1 & $R_{b 1}, R_{b 2}, \textcolor{red}{R_f}, R_{c 1}, R_{c 2}, \textcolor{red}{\beta}$ & \begin{tabular}{l}
\begin{tabular}{c}
{$[50,25,\textcolor{red}{0.5},1.2,0.25,\textcolor{red}{30}]$,} \\
    {$[150,70,\textcolor{red}{3},2.5,1.20,\textcolor{red}{50}]$}
\end{tabular}
\end{tabular} \\
\hline 2 & $M, S, V_0, \textcolor{red}{k}, \textcolor{red}{P_0}, T, T_0$ & \begin{tabular}{l}
$[30,0.005,0.002,\textcolor{red}{1000},\textcolor{red}{90000},290,340]$, \\
$[60,0.02,0.01,\textcolor{red}{5000},\textcolor{red}{110000},296,360]$
\end{tabular} \\
    \hline 3 & $T_u, H_u, \textcolor{red}{H_l}, r, \textcolor{red}{r_w}, T_l, L, K_w$ & \begin{tabular}{l}
\begin{tabular}{c}
{$[63070,990,\textcolor{red}{700},100,\textcolor{red}{0.05},63.1,1120,9855]$,} \\
{$[115600,1110,\textcolor{red}{820},50000,\textcolor{red}{0.15},116,1680,12045]$}
\end{tabular}

\end{tabular} \\
\hline
\end{tabular}
\caption{Analytical test functions input descriptions}
\label{table:config}
\end{table}
The quantitative variables are presented in black, and the categorical variables are highlighted in red. The range for each variable is carefully defined to represent realistic operating conditions and ensure meaningful analysis. The categorical variables are discretized into four equally spaced categories.

\subsubsection{Real world test function}
Beam Bending problem is a non-equally spaced real-world problem \citep{zhang2020latent}. It is a classical engineering problem where the categorical variable is the cross-sectional shape of the beam, which has five categories: circular, square, I-shape, hollow square, and hollow circular. The categories are not equally spaced due to the varying moments of inertia $I(t)$ associated with each shape. Table \ref{tab:table 3} is a summary of the cross-sectional shapes and their corresponding normalized moments of inertia. As shown in the table, the normalized moments of inertia $I(t)$ vary significantly across different cross-sectional shapes, resulting in non-equally spaced categories. This non-uniform spacing reflects the distinct impact each shape has on the beam's deformation.
\begin{table}[]
\renewcommand{\arraystretch}{1.5}

    \centering
    \begin{tabular}{cccccc}
\hline$I(t)$ & Circular & Square & I-shape & Hollow Square & Hollow Circular \\
\hline Value & 0.0491 & 0.0833 & 0.0449 & 0.0633 & 0.0373\\
\hline
\end{tabular}
    \caption{Beam bending problem}
    \label{tab:table 3}
\end{table}

\subsection{Test functions for optimization tasks.}
\label{subsec:optimization_test}
\subsubsection{Synthetic Test Functions}
\label{sec:syntest}
We use several synthetic test functions: \textbf{Func-2C}, \textbf{Func-3C}, and \textbf{Ackley-cC}, whose input spaces comprise both continuous and categorical variables. Each of the categorical inputs in the three test functions has multiple values.

\begin{itemize}
    \item \textbf{Func-2C} is a test problem with 2 continuous inputs $(d=2)$ and 2 categorical inputs $(c=2)$. The categorical inputs decide the linear combinations between three 2-dimensional global optimisation benchmark functions: \textbf{Beale (bea)}, \textbf{Six-Hump Camel (cam)}, and \textbf{Rosenbrock (ros)}.
    \item \textbf{Func-3C} is similar to Func-2C but with 3 categorical inputs $(c=3)$, which leads to more complicated linear combinations among the three functions.
    \item \textbf{Ackley4C} includes $c = \{4\}$ categorical inputs and 3 continuous inputs $(d=3)$. The categorical dimensions are transformed into 3 categories.
\end{itemize}

The value ranges for both continuous and categorical inputs of these functions are summarised in Table \ref{table:synthetic}.

\begin{table}[h]
\renewcommand{\arraystretch}{1.5}
\centering
\begin{tabular}{c|c|c}
\hline
\textbf{Function} & \textbf{Inputs} $\mathbf{z}=[\mathbf{h}, \mathbf{x}]$ & \textbf{Input values} \\
\hline
\text{Func2C} & $h_1$ & $\{\operatorname{ros}(\mathbf{x}), \operatorname{cam}(\mathbf{x}), \operatorname{bea}(\mathbf{x})\}$ \\
(d=2, c=2) & $h_2$ & $\{+\operatorname{ros}(\mathbf{x}),+\operatorname{cam}(\mathbf{x}),+\operatorname{bea}(\mathbf{x}),+\operatorname{bea}(\mathbf{x}),+\operatorname{bea}(\mathbf{x})\}$ \\
& $\mathbf{x}$ & $[-1,1]^2$ \\
\hline
\text{Func3C} & $h_1$ & $\{\operatorname{ros}(\mathbf{x}), \operatorname{cam}(\mathbf{x}), \operatorname{bea}(\mathbf{x})\}$ \\
(d=2, c=3) & $h_2$ & $\{+\operatorname{ros}(\mathbf{x}),+\operatorname{cam}(\mathbf{x}),+\operatorname{bea}(\mathbf{x}),+\operatorname{bea}(\mathbf{x}),+\operatorname{bea}(\mathbf{x})\}$ \\
& $h_3$ & $\{+5 \times \operatorname{cam}(\mathbf{x}),+2 \times \operatorname{ros}(\mathbf{x}),+2 \times \operatorname{bea}(\mathbf{x}),+3 \times \operatorname{bea}(\mathbf{x})\}$ \\
& $\mathbf{x}$ & $[-1,1]^2$ \\
\hline
\text{Ackley4C} & $h_i$ \text{ for } $i=1,2, 3, 4$ & $h_i\in\{0,0.5,1\}$ \\
(d=3, c=4) & $\mathbf{x}$ & $[-1,1]^3$ \\
\hline
\end{tabular}
\caption{Input descriptions for the synthetic test functions}
\label{table:synthetic}
\end{table}

\subsubsection{MLP Hyperparameter Tuning}
We defined a real-world task of tuning hyperparameters for an MLP (Multi-layer Perceptron) regressor on the \textbf{California Housing} dataset. This problem involves 3 categorical inputs and 3 continuous inputs. The output is the negative mean squared error (MSE) on the test set of the \textbf{California Housing} dataset.

The MLP model is built using \texttt{MLPRegressor} from \texttt{scikit-learn}, and the hyperparameters include the following:

\begin{itemize}
    \item \textbf{Activation function (ht1)}: This categorical input takes values from \{\texttt{logistic}, \texttt{tanh}, \texttt{relu}\}.
    \item \textbf{Learning rate schedule (ht2)}: This categorical input includes \{\texttt{constant}, \texttt{invscaling}, \texttt{adaptive}\}.
    \item \textbf{Solver (ht3)}: This categorical input can take one of \{\texttt{sgd}, \texttt{adam}, \texttt{lbfgs}\}.
    \item \textbf{Hidden layer size ($x_1$)}: This continuous input varies between 1 and 100.
    \item \textbf{Regularization parameter (alpha, $x_2$)}: This continuous input lies in the range $[10^{-6}, 1]$.
    \item \textbf{Tolerance for optimization ($x_3$)}: This continuous input varies between 0 and 1.
\end{itemize}

The value ranges for both continuous and categorical inputs for this problem are shown in Table \ref{table:realworld}.

\begin{table}[h]
\renewcommand{\arraystretch}{1.5}
\centering
\begin{tabular}{c|c|c}
\hline
\textbf{Problems} & \textbf{Inputs} $\mathbf{z}=[\mathbf{h}, \mathbf{x}]$ & \textbf{Input values} \\
\hline
\text{MLP-CaliHousing} & \text{activation function (ht1)} & $\{\text{logistic, tanh, relu}\}$ \\
(d=3, c=3) & \text{learning rate (ht2)} & $\{\text{constant, invscaling, adaptive}\}$ \\
& \text{solver (ht3)} & $\{\text{sgd, adam, lbfgs}\}$ \\
& $x_1$ & $[1, 100]$ \\
& $x_2$ (\text{alpha}) & $[10^{-6}, 1]$ \\
& $x_3$ (\text{tolerance}) & $[0, 1]$ \\
\hline
\end{tabular}
\caption{Input ranges for the real-world problem}
\label{table:realworld}
\end{table}

\subsubsection{SVM Hyperparameter Tuning}

SVM hyperparameter tuning task in this paper involves tuning two continuous hyperparameters and one categorical hyperparameter for a Support Vector Regressor (SVR) on the \textbf{California Housing} dataset. The SVR model is implemented using \texttt{SVR} from \texttt{scikit-learn} and is evaluated via cross-validation. The optimization objective is to minimize the logarithm of the mean squared error (log MSE) on the dataset.

The hyperparameters include:

\begin{itemize} \item \textbf{Regularization parameter ($C$)}: This continuous input takes values from $[0.1, 100]$. \item \textbf{Epsilon ($\epsilon$)}: This continuous input varies between $0.1$ and $100$. \item \textbf{Kernel type}: This categorical input can take one of the following values: {\texttt{poly}, \texttt{rbf}, \texttt{sigmoid}, \texttt{linear}}. \end{itemize}

The input ranges for this real-world task are summarized in Table \ref{table:svm}.

\begin{table}[h]
\renewcommand{\arraystretch}{1.5}
\centering 
\begin{tabular}{c|c|c} 
\hline \textbf{Problems} & \textbf{Inputs} $\mathbf{z}=[\mathbf{h}, \mathbf{x}]$ & \textbf{Input values} \\
\hline 
\text{SVR} & $C$ & $[0.1, 100]$ \\
(d=2, c=1) & $\epsilon$ & $[0.1, 100]$ \\
& \text{Kernel type} & ${\text{poly, rbf, sigmoid, linear}}$ \\
\hline 
\end{tabular} 
\caption{Input ranges for the SVR hyperparameter tuning problem}
\label{table:svm}
\end{table}

\end{document}